\pdfoutput=1
\documentclass[11pt,a4paper]{article}
\usepackage[utf8]{inputenc}
\usepackage{amsfonts, amsmath, amssymb, color, a4wide}
\usepackage[utf8]{inputenc} 
\usepackage[T1]{fontenc}    
\usepackage{url}            
\usepackage{booktabs}       
\usepackage{amsfonts}       
\usepackage{nicefrac}       
\usepackage{microtype}      
\usepackage{fancybox,framed}
\usepackage{lmodern}
\usepackage{authblk}

\usepackage{amsfonts}
\usepackage{amsmath}
\usepackage{amssymb}
\usepackage{bm}
\usepackage{algorithm}
\usepackage{algorithmic}
\usepackage{mathtools}
\usepackage{xspace}
\usepackage{xcolor}
\usepackage{amsthm}
\usepackage{cases}
\usepackage{dsfont}
\usepackage{bbm}
\RequirePackage[shortlabels]{enumitem}
\usepackage{accents}
\usepackage[numbers]{natbib}

\newtheorem{theorem}{Theorem}[section]

\newtheorem{lemma}[theorem]{Lemma}

\DeclarePairedDelimiter\floor{\lfloor}{\rfloor}

\newtheorem{assumption}{Assumption}

\input{macros_gilles.tex}

\newcommand{\vertiii}[1]{{\left\vert\kern-0.25ex\left\vert\kern-0.25ex\left\vert #1 
		\right\vert\kern-0.25ex\right\vert\kern-0.25ex\right\vert}}

\title{Active Ranking of Experts\\
	Based on their Performances in Many Tasks}
\date{}

\author[1]{El Mehdi Saad}
\author[1]{Nicolas Verzelen}
\author[2]{Alexandra Carpentier}

\affil[1]{INRAE, Mistea, Institut Agro, Univ Montpellier, Montpellier, France.}
\affil[2]{Institut für Mathematik, Universität Potsdam, Germany.}

\begin{document}
	\maketitle

	\begin{abstract}
		We consider the problem of ranking $n$ experts based on their performances on $d$ tasks. We make a monotonicity assumption stating that for each pair of experts, one outperforms the other on all tasks. We consider the sequential setting where in each round, the learner has access to noisy evaluations of actively chosen pair of expert-task, given the information available up to the actual round. Given a confidence parameter $\delta \in (0,1)$, we provide strategies allowing to recover the correct ranking of experts and develop a bound on the total number of queries made by our algorithm that hold with probability at least $1-\delta$. We show that our strategy is adaptive to the complexity of the problem (our bounds are instance dependent), and develop matching lower bounds up to a poly-logarithmic factor. Finally, we adapt our strategy to the relaxed problem of best expert identification and provide numerical simulation consistent with our theoretical results.
	\end{abstract}
	
	\section{Introduction}

	Consider the problem of ranking $n$ experts based on noisy evaluations of their performances in $d$ tasks. This problem arises in many modern applications such as recommender systems \citep{zhou2012state} and crowdsourcing \citep{snow2008cheap,raykar2010learning, karger2011iterative,lu2015recommender}, where the objective is to recommend, for instance, films, music, books, etc based on the product ratings. Sports is another field where ranking plays an important role via the task of player ranking based on their data-driven performances \citep{morgulev2018sports,pappalardo2019playerank}. In many situations, it is possible to rank experts in an active fashion by sequentially auditing the performance of a chosen expert on a specific task, based on the information collected previously. 
	
	In this paper, we consider such a sequential and active ranking setting. For a positive integer $N$ we write $\intr{N}=\{1, \dots, N\}$. We consider that the performance of each expert $i \in \intr{n}$ on task $j \in \intr{d}$ is modeled via random variable following an unknown $1$-sub-Gaussain distribution $\nu_{ij}$ 
	with mean $M_{ij}$ - and the matrix $M = (M_{i,j})_{i,j}$ encodes the mean performance of each expert on each task. Samples are collected sequentially in an active way: at each round $t$, the learner chooses a pair expert-task and receives a sample from the corresponding distribution - and we assume that the obtained samples are mutually independent conditional on the chosen distribution.
	
	Our setting is related to the framework of comparison-based ranking considered in the active learning literature \citep{jamieson2011active, heckel2019active}. This literature stream is tightly connected to the \textit{dueling bandit setting} \citep{yue2012k, ailon2014reducing}, where the learner receives binary and noisy feedback on the relative performance of any pair of experts - which is a specific case of our setting when $n = d$. An important decision for setting the ranking problem is then on defining a notion of order between the experts based on the mean performance $M$ - which is ambiguous for a general $M$. In the active ranking literature, as well as in the dueling bandit literature, it is customary to define a notion of order among any two experts, for instance, through a score as e.g.~the \textit{Borda score}, or through another mechanism as e.g.~Condorcet winners. E.g.~the widely used \textit{Borda scores} corresponds, for each expert $i$, to the average performance of expert $i$ on all questions - so that this corresponds to ranking experts based on their average performance across all tasks. 
	We discuss this literature in more detail in Section~\ref{sec:relwork}.

	In this paper, we are inspired by recent advances in the batch literature on this setting - where we receive one sample for each entry expert-task pair. Beyond the historical parametric Bradley-Luce-Terry (BLT) model~\cite{bradley1952rank}, recent papers have proposed batch ranking algorithms in more flexible settings, where no parametric assumption is made, but where a monotonicity assumption, up to some unknown permutation, is made on the matrix $M$. A minimal such assumption was made in~\cite{flammarion2019optimal} where it is assumed that the matrix $M$ is monotone up to a permutation of its lines, namely of the experts. 
	More precisely, they suppose that there is a complete ordering of the $n$ experts, and that expert $i$ is better than expert $k$ if the first outperforms the latter in all tasks:
	\begin{equation}\label{eq:order}
		``~i \quad \accentset{\bullet}{\ge} \quad k~": \quad  \forall j \in \intr{d},~ M_{ij} \ge M_{kj}.
	\end{equation}
	Such a shape constraint is arguably restrictive, yet realistic in many applications - e.g.~when the tasks are of varying difficulty yet of same nature. There have been recently very interesting advances in the batch literature on this setting, through the construction and analysis of efficient algorithm that fully leverage such a shape constraint. And importantly, these approaches vastly overperform on most problems with a naive strategy that would just rank the experts according to their Borda scores. Beyond the Borda score, this remark is true for any fixed score, as the recent approaches mentioned ultimatively resort in adapting the computed scores to the problem in order to put more emphasis on informative questions.
	See Section~\ref{sec:relwork} for an overview of this literature.
	
	
	In this paper, we start from the above remark in the batch setting - namely that in the batch setting and under a monotonicity assumption, ranking according to fixed scores as e.g.~Borda is vastly sub-optimal - and we aim at exploring whether this is also the case in the sequential and active setting. We therefore make the monotonicity assumption of Equation~\eqref{eq:order}, and aim at recovering the exact ranking in the active and online setting described above. More precisely, given a confidence parameter $\delta \in (0,1)$, our objective is to rank all the experts using as few queries as possible and consequently adapt to the unknown distributions $(\nu_{ij})_{ij}$. 
	
	In this paper, we make the following contributions: First,  In Section~\ref{sec:two_e}, we consider the problem of comparing two expert ($n=2$) based on their performances on $d$ tasks under a monotonicity assumption. We provide a sampling strategy and develop distribution dependent upper bound on the total number of queries made by our algorithm to output the correct ranking with probability at least $1-\delta$, $\delta \in (0,1)$ being prescribed confidence.
	We then consider the problem of ranking estimation for a general $n$ in Section~\ref{sec:gen_e}, we use the previous algorithm for pairwise comparison as a building block and provide theoretical guarantees on the correctness of our output and a bound on the total number of queries made that holds with high probability.
	Next, we consider the relaxed objective of identifying the best expert out of $n$, we provide a sampling strategy for this problem and a bound on the query budget that holds with high probability. 
	In Section~\ref{sec:low} we give some instance-dependent lower bounds for the three problems above, showing that all our algorithms are optimal up to a poly-logarithmic factor. In Section~\ref{sec:simu} of the appendix, we make numerical simulations on synthetic data. The proofs of the theorems are in the appendix.

	\section{Problem formulation and notation}
	
	Consider a set of $n$ experts evaluated on $d$ tasks. As noted in the introduction, the performance of each expert $i \in \intr{n}$ on task $j\in \intr{d}$, is modeled via a random variable following an unknown distribution $\nu_{ij}$ with mean $M_{ij}$ - and write $M = (M_{ij})_{i\le n, j\le d}$. We refer to $\nu = (\nu_{ij})_{i\le n, j\le d}$ as the (expert-task) performance distribution, and to $M = (M_{ij})_{i\le n, j\le d}$ as the mean (expert-task) performance. We have two aims in this work: (i) Ranking identification $\bm{(R)}$, namely identifying the permutation that ranks the expert, and (ii) Best-expert identification $\bm{(B)}$, namely identifying the best expert - we will define precisely these two objectives later.
	
	On top of assuming as explained in the introduction the existence of a total ordering of experts following \eqref{eq:order}, we also make an identifiability assumption, so that we have a unique assumption for our problems of ranking $\bm{(R)}$ and best-expert identification $\bm{(B)}$; this is summarised in the following assumption.
	\begin{assumption}\label{assump}
		Suppose that the following assumption holds:
		\begin{itemize}
			\item \textbf{Monotonicity:} there exists a permutation $\pi:\intr{n} \to \intr{n}$, such that $\forall i \in \intr{n-1}, \forall j \in \intr{d}: M_{\pi(i)j} \ge M_{\pi(i+1)j}$.
			\item \textbf{Bounded mean performance:} we assume that the mean performance $M$ takes value in $[0,1]$, namely $M_{i,j} \in [0,1]$ for all $i\in \intr{n}$ and $j\in \intr{d}$.
		\end{itemize} 
		We will then assume one of these two identifiability assumptions, depending on whether we are considering the ranking identification problem $\bm{(R)}$, or the best expert identification problem $\bm{(B)}$.
		\begin{itemize}
			\item \textbf{Identifiability for $\bm{(R)}$:} for $i, k \in \intr{n}$, if $i\neq k$ then $\exists j \in \intr{d}: M_{i,j} \neq M_{k,j}$.
			\item \textbf{Identifiability for $\bm{(B)}$:} for $i \in \{2, \dots, n\}$, $\exists j \in \intr{d}: M_{\pi(i),j} \neq M_{\pi(1),j}$. 
		\end{itemize}
	\end{assumption}
	Note that under the identifiability assumption for $\bm{(R)}$, there exists a unique ranking, and that under the identifiability assumption for $\bm{(B)}$, there exists a unique best expert (and that the identifiability assumption for $\bm{(R)}$ implies the the identifiability assumption for $\bm{(B)}$).
	
	We write $\pi:\intr{n} \to \intr{n}$ for the corresponding permutation such that: $\pi(1)~  \accentset{\bullet}{\ge} ~ \pi(2) ~ \accentset{\bullet}{\ge} ~ \dots ~ \accentset{\bullet}{\ge}~ \pi(n)$. In what follows, we write $\cM_{n,d} \subset \mathbb{R}^{n\times d}$ for the set of matrices $M$ satisfying Assumption~\ref{assump}. Moreover, we are also going to assume in what follows that the samples collected by the learner are sub-Gaussian.
	\begin{assumption}\label{assump:sg}
		For each $i \in \intr{n}, j \in \intr{d}$: $\nu_{ij}$ is $1$-sub-Gaussian.
	\end{assumption}
	It is satisfied for e.g.~random variables taking values in $[0,1]$, or for Gaussian distributions of variance bounded by $1$. 
	
	

	We consider the \textit{fixed confidence setting} presented in the sequential and active identification literature \citep{heckel2019active,kaufmann2016complexity, garivier2016optimal}. At each time $t$, the learner tasks an expert on a task - we write $(I_t, J_t)$ for this pair of expert and task -  based on previous observations, and receive an independent sample $Z_t$ following the distribution of $\nu_{I_tJ_t}$.
	Based on this information, the learner also decides whether it terminates, or continues sampling - and we write $N$ for the termination time, to which we also refer to as number of queries. Upon termination, the learner outputs an estimate, and we consider here the two problems of ranking $\bm{(R)}$, and best-expert identification $\bm{(B)}$: 
	\vspace{-3mm}
	\begin{itemize}
		\item{(i)} Ranking identification $\bm{(R)}$: the learner aims at outputting a ranking $\hat{\pi}$ that estimates $\pi$. For a given confidence parameter $\delta \in (0,1)$, we say that it is $\delta$-accurate for ranking $\bm{(R)}$ if it satisfies:
		\begin{equation*}
			\bm{(R)}: \quad \mathbb{P}( \hat{\pi} \neq \pi) \le \delta
		\end{equation*}
		\item{(ii)} Best-expert identification $\bm{(B)}$: the learner aims at outputting an expert $\hat{k} \in \intr{d}$ that estimates the best performing expert, namely $\pi(1)$. For a given confidence parameter $\delta \in (0,1)$, we say that it is $\delta$-accurate for best-expert identification $\bm{(B)}$ if it satisfies:
		\begin{equation*}
			\bm{(B)}: \quad \mathbb{P}( \hat{k} \neq \pi(1)) \le \delta.
		\end{equation*}
	\end{itemize}
	\vspace{-3mm}
	The performance of any $\delta$-accurate algorithm is then measured through the total number $N$ of queries made when the procedure terminates. The emphasis is then put on developing high probability guarantees on $N$ i.e., bounds on $N$ on the event of probability at least $1-\delta$ where the $\delta$-accurate algorithm is correct .
	
	


	\paragraph{Notation:} Let $A \in \mathbb{R}^{n\times d}$ be a rectangular matrix, for $i \in \intr{n}$, we write $A_i$ for its $i^{\text{th}}$ line. Let $\norm{.}_2$ denote the euclidean norm and $\norm{.}_1$ denotes the $l_1$ norm on $\mathbb{R}^d$. For two numbers $x,y$, denote $x \vee y = \max\{x,y\}$.
	
	\section{Related work}\label{sec:relwork}
	
	\paragraph{Best-arm identification and Top-k bandit problems:} Active ranking is related to many works in the vast literature of identification in the multi-armed bandit (MAB) model \citep{mannor2004sample, kaufmann2016complexity, garivier2016optimal}, where each arm is associated to a univariate random variable. The learner's objective is to build a sampling strategy from the arms' distributions to identify the one with the largest mean - which would resemble our objective of best expert identification. Other related works consider the more general objective of identifying the top-$k$ arms with the largest means - and in the case where the problem is solved for all $k$, it resembles the ranking problem. In this work, we consider a more general setting where instead of having a univariate distribution that characterizes the performance of an expert (akin to an arm in the aforementioned literature), we have a multivariate distribution, corresponding to the $d$ questions.\\
	Many of the previous works rely on a successive elimination approach by discarding arms that are seemingly sub-optimal. This idea is not directly applicable in our setting due to the multi-dimensional aspect of the rewards/performances of each candidate expert. Perhaps, the most natural way to cast our setting into the standard MAB framework is to associate each expert with the average of its performances in all tasks. For expert $i \in \intr{n}$, let $Y_i := \sum_{j=1}^{d} X_{ij}/d$ and  $m_i := \mathbb{E}[Y_i]$ and observe that due to the monotonicity assumption on the matrix $M$, the ranking experts with respect to the means $m_i$ leads to the correct ranking. Moreover, the learner has access to samples of $Y_i$ by sampling the performance of expert $i$ in a task chosen uniformly at random from $d$. Using the last scheme, one can exploit MAB methods to recover the correct ranking. However, we argue that such methods are sub-optimal: consider the problem of comparing two experts associated with distributions $Y_1$ and $Y_2$ defined previously. The minimum number of queries necessary to decide the better expert is of order $\cO(1/(m_1 - m_2)^2) = \cO(d^2/(\sum_{j=1}^{d} M_{1j}-M_{2j})^2)$, in contrast, the procedure presented in Section~\ref{sec:two_e}, decides the optimal expert  using at most $\cO(d/\sum_{j=1}^{d} (M_{1j} - M_{2j})^2)$ up to logarithmic factors. We can show using Jensen's inequality that our bound matches in the worst case the former. However, in general, the improvement can be up to a factor of $d$. This is due to the fact that our strategy uses a more refined choice of tasks to sample from.

	\paragraph{On comparison-based ranking algorithms and duelling bandits:} Many previous works consider the problem of ranking based on comparisons between experts in the online learning literature \citep{ailon2014reducing,chen2020combinatorial,heckel2019active,jamieson2011active,jamieson2015sparse,urvoy2013generic,yue2012k}. For instance, in \citep{heckel2019active}, the authors consider a setting where data consist of noisy binary outcomes of comparisons between pairs of experts. More precisely, the outcome of a queried pair of expert $(i,j)\in \intr{n}\times \intr{n}$ is a Bernoulli variable with mean $M_{ij}$. The experts are ranked with respect to their Borda scores defined for expert $i$ by $\tau_i := \sum_{k \in \intr{n}\setminus \{i\}} M_{ik}/(n-1)$. The authors provide a successive elimination-based procedure leading to optimal guarantees up to logarithmic factors. Their bound on the total number of queries to recover the correct ranking is of order $\tilde{\cO}(\sum_{i=1}^{n-1}/(\tau_{\pi(i)} - \tau_{\pi(i+1)})^2)$, where $\pi$ is the permutation corresponding to the correct ranking. Their setting can be harnessed into ours by considering that $d=n$ and that for each expert $i$, the task $j$ consists of outperforming expert $j$. 
	However, contrary to duelling bandit, our model underlies the additional assumption that experts' performances are monotone which is common in the batch ranking literature -- see the next paragraph. In the duelling bandit framework, our monotonicity assumption is equivalent to the strong stochastic transitivity assumption~\cite{shah2016stochastically}, in the sense that $M_{\pi(i+1),k}\leq M_{\pi(i),k}$ for all $k$. In this work, our main idea is to build upon the monotonoticity assumption to drastically reduce in some problem instances the number of queries. Existing approaches based on fixed scores cannot be optimal on all problem instances as they do not adapt to the problem -  e.g.~applying Borda-scores algorithms in our setting leads to a total number of queries of order $ d^2/(\sum_{j=1}^{d}M_{1,j} - M_{2,j})^2$, which compared to our bound is sub-optimal, with a difference up to a factor $d$.
	
	\paragraph{On batch ranking:} 
	In Batch learning, the problem of ranking has attracted a lot of attention since the seminal work of~\citealp{bradley1952rank}. In this setting, the learner either observes noisy comparisons between experts or the performances of experts in given tasks. Observing that parametric models such as Bradley-Luce-Terry are sometimes too restrictive~\citealp{shah2016stochastically} have initiated a stream of literature on non-parametric models under shape constraints~\cite{flammarion2019optimal, shah2020permutation, mao2018breaking,liu2020better, shah2019feeling, pananjady2020worst,pananjady2022isotonic, emmanuel2022optimal}. Our monotonicity assumption inspired from~\citealp{flammarion2019optimal} is the weakest one in this literature. That being said, our results and methods differ importantly from this literature as we aim at recovering the true ranking with a sequential strategy while these works aim at estimating an approximate ranking according to some loss function in the batch setting. 
	

	\paragraph{Link to adaptive signal detection:} Consider two expert ($n=2$) and let $M_1, M_2 \in \mathbb{R}^d$ be their mean performance vectors on $d$ tasks. A closely related problem to the objective of identifying the best expert out of the two (under monotonicity assumption) is signal detection performed on the differences vector $\Delta = M_1 - M_2$, where the aim is to decide whether $\{M_1 = M_2 \}$ or $\{ M_1 \neq M_2\}$. There is a vast literature in the batch setting for the last testing problem \cite{baraud2002non, ingster2012minimax, poor1998introduction}.\cite{castro2014adaptive} considered the signal detection problem in the active setting: given a budget $T$, the learner's objective is to decide between the two hypotheses $\{\Delta = 0\}$ and $\{\Delta \neq 0\}$, under the assumption that $\Delta$ is $s$-sparse. When the magnitude of non-zero coefficients is $\mu$, they proved that reliable signal detection requires a budget of order $\frac{d\log(1/\delta)}{s\mu^2}$, where $\delta$ is a prescribed bound on the risk of hypothesis testing. Our theoretical guarantees are consistent with the last bound and are valid for any difference vector $\Delta$.
	
	\paragraph{Link to bandits with infinitely many arms:} As discussed earlier, a key feature of our setting, is the ability of the learner to pick the task to assess the performance of chosen experts. When comparing two experts, since the tasks underlying the greatest performance gaps are unknown, the learner should balance between exploring different tasks and committing to a small number of tasks to build a reliable estimate of the underlying differences. This type of breadth versus depth trade-off in pure exploration arises in the context of best-arm identification in the infinitely-many armed bandit problem. It was introduced by \cite{berry1997bandit} and analysed in many subsequent works (\citealp{jamieson2016power, aziz2018pure, katz2020true, de2021bandits}). While comparing two experts in our setting includes dealing with similar challenges in the previous literature, note that we are particularly interested in detecting the existence (and the sign) of the gaps between experts' performances rather than identifying tasks with the largest performance difference.

	\section{Comparing two experts ($n=2$)}\label{sec:two_e}
	

	We start by considering the case where $n=2$, which we will then use as a building block for the general case. In this case, the ranking problem $\bm{(R)}$ is equivalent to the best-expert identification problem $\bm{(B)}$. 
	Algorithm~\ref{algo1} takes as input two parameters: a confidence level $\delta \in(0,1)$, and a precision parameter $\epsilon>0$, and outputs the best expert if the $L_2$ distance between the compared experts is greater than $\epsilon$.  
	
	One wants ideally to focus on the task displaying the greater gap in performance in order to quickly identify the best expert, however, such knowledge is not available to the learner. This raises the challenge of balancing between sampling as many different tasks as possible to pick one that has a large gap and focusing on one task to be able to distinguish between the two experts based on this task.        
	Such width versus depth trade-off arises in many works of best arm identification with infinitely many arms \citep{de2021bandits, jamieson2016power}, where the proportion of optimal arms is $p \in (0,1)$ and the gap between optimal and sub-optimal arms is $\Delta>0$. In contrast, the gaps between tasks (equivalent to arms in the previous problem) may be dense, in order to bypass this difficulty we make the following observation: For $j \in \intr{d}$ let $x_j$ denote the gap in task $j$: $x_j := \abs{M_{1,j} - M_{2,j}}$. Denote $x_{(j)}$ the corresponding decreasing sequence, we have:
	\begin{equation}\label{eq:approx}
		\max_{s \in \intr{d}} sx^2_{(s)} \le \sum_{j=1}^{d} x^2_{j} \le \log(2d) \max_{s \in \intr{d}} sx^2_{(s)}.
	\end{equation}
	The last inequality was presented in \cite{audibert2010best} in the context of fixed budget best arm identification (see Lemma~\ref{lem:sparse} in the appendix). It suggests that, up to a logarithmic factor in $d$, the $L_2$ distance is mainly concentrated in the top $s^*$ elements with the largest magnitude, where $s^*$ is an element in $\intr{d}$ satisfying the maximum in equation \eqref{eq:approx}. We exploit this observation by focusing our exploration effort on tasks with the largest $s^*$ gaps.
	
	In the first part of our strategy, presented in Algorithm~\ref{algo1}, we discretize the set of possible values of $s^*$ and $x_{(s^*)}$ using a doubling trick. The last scheme was adapted from \cite{jamieson2016power}, and allows to be adaptive to the unknown quantities $(s^*, x_{(s^*)})$. In the second part, given a prior guess $(s,h)$, we run Algorithm~\ref{algo3} consisting of two main ingredients: (i) sampling strategy and (ii) stopping rule. We start by sampling a large number of tasks (with replacement) to ensure that with large probability, a proportion $s/d$ of ``good tasks" - i.e.~relevant for ranking the two experts, namely the tasks where the experts differ the most - where chosen, then we proceed by median elimination by keeping only half of the tasks with the largest empirical mean in each iteration and doubling the number of queries made to have a more precise estimate of the population means. The aim of this process is to focus the sampling force on the ``good tasks", by gradually eliminating tasks where the two experts perform similarly. Also, at each iteration, we run a test on the average of the kept tasks to potentially conclude on which one of the two experts is best and terminate the algorithm.

	\begin{algorithm}[!ht] 
		\caption{\texttt{Compare}$(\delta, \epsilon)$  \label{algo1} }
		\begin{algorithmic}
			\STATE \textbf{Input}: $\delta, \epsilon$.
			\STATE \textbf{Initialize}: $\rho=1$, output $\hat{k} = \texttt{null}$.
			\WHILE{ $\{\hat{k} = \texttt{null} \}$ and $\epsilon^2< 4\log(2d)d~2^{-\rho} $}
			\FOR{$r = 0,\dots,\rho$}
			\STATE \textbf{Set}  $s_{r}=2^rd/2^{\rho}$, $h_{r} = 1/\sqrt{2^r}$. 
			\STATE \textbf{Run}: $\texttt{Try-compare}(\delta/(10\rho^3\log(d)), s_{k},h_{k})$ and $\textbf{Set}:$ $\hat{k}$ its output.
			\STATE If $\hat{k} \neq \texttt{null}$, \textbf{break}.
			\ENDFOR
			\STATE $\rho \gets \rho +1$
			\ENDWHILE 
			\STATE \textbf{Output:} $\hat{k}$.
		\end{algorithmic}
	\end{algorithm}

	\begin{algorithm}[!ht] 
		\caption{\texttt{Try-compare}($\delta, s,h $)  \label{algo3} }
		\begin{algorithmic}
			\STATE \textbf{Input}: $\delta \in (0,1), s \le d, h >0$.
			\STATE \textbf{Initialize}: $\phi = 2^{\lceil \log_2[26\log(1/\delta)d/s]\rceil}$, $n_0=64/h^2$, $\hat{k} \gets \texttt{null}$.
			\STATE \textbf{Draw} a set (denoted $S_1$) of $\phi$ elements from $\intr{d}$ uniformly at random \textit{with replacement}.
			\STATE \textbf{Initialize}: $S_{1}^{(12)} = S_1$ and $S_1^{(21)} = S_1$.
			\FOR{$\ell = 1, \dots, \lceil \log_{4/3}(d/s)\rceil +1$}
			\STATE \textbf{Sample}: For each element $j \in S^{(12)}_{\ell} \cup S_{\ell}^{(21)}$, sample the entries $(1,j)$ and $(2,j)$ $t_{\ell} = n\phi/\abs{S^{(12)}_{\ell}}$ times, denote $(X_{1,j}^{(u,\ell)}, X_{2,j}^{(u,\ell)})_{u \le t_{\ell}}$ the obtained samples.
			\STATE \textbf{Compute}: the means 
			\begin{align*}
				\hat{\mu}_{j,\ell}^{(12)} &= \frac{1}{t_{\ell}} \sum_{u=1}^{t_{\ell}} \left( X_{1,j}^{(u,\ell)}- X_{2,j}^{(u,\ell)} \right)\\
				\hat{\mu}_{\ell}^{(12)} &= \frac{1}{n_0\phi} \sum_{u=1}^{t_{\ell}}\sum_{j\in S_{\ell}} \left( X_{1,j}^{(u,\ell)}- X_{2,j}^{(u,\ell)} \right).
			\end{align*}
			\STATE Denote $\hat{\mu}_{\ell}^{(21)} = - \hat{\mu}_{\ell}^{(12)}$
			\IF{ $\hat{\mu}_{\ell}^{(12)} \ge \sqrt{\frac{2\log(2/\delta)}{n_0\phi}}$}
			\STATE $\hat{k} = 1$, \textbf{break}
			\ELSIF{$\hat{\mu}_{\ell}^{(21)} \ge \sqrt{\frac{2\log(2/\delta)}{n_0\phi}}$}
			\STATE $\hat{k} = 2$, \textbf{break}
			\ELSE
			\STATE Find the median of $(\hat{\mu}_{j, \ell}^{(12)})_{i \in S^{(12)}_{\ell}}$ and $(\hat{\mu}_{j, \ell}^{(21)})_{i \in S^{(21)}_{\ell}}$, denoted $m^{(12)}_{\ell}$ and $m_{\ell}^{(21)}$ resp.
			\STATE $S^{(12)}_{\ell+1} \gets S^{(12)}_{\ell} \setminus \{i \in S^{(12)}_{\ell}: \hat{\mu}_{j,\ell}^{(12)} <m^{(12)}_{\ell}\}$.
			\STATE $S^{(21)}_{\ell+1} \gets S^{(21)}_{\ell} \setminus \{i \in S^{(21)}_{\ell}: \hat{\mu}_{j,\ell}^{(21)} <m^{(21)}_{\ell}\}$.
			\ENDIF
			\ENDFOR
			\STATE \textbf{Output:} $\hat{k}$ 
		\end{algorithmic}
	\end{algorithm}
	
	We now turn to the analysis of the performance of $\texttt{compare}$. We first show that with probability at least $1-\delta$, Algorithm~\ref{algo1} does not make the wrong diagnostic. Based on the precision parameter $\epsilon$, we prove that if the unknown squared $L_2$ distance between the experts' performance vectors is larger than $\epsilon$, then the algorithm identifies the optimal expert out of the two, with a large probability. We also bound on the same event the total number of queries $N$ made by the procedure. 
	
	\begin{theorem}\label{thm:up2}
		Suppose Assumption~\ref{assump} holds. For $\delta \in (0,1)$, $\epsilon > 0$, consider Algorithm~\ref{algo1} with input $(\delta, \epsilon)$. Define
		\[
		H := \frac{d}{\norm{M_1 - M_2}_2^2}\ .
		\]
		With probability at least $1-\delta$, we have
		\begin{itemize}
			\item The output $\hat{k}$ satisfies: $\hat{k} \in \{ \texttt{null}, \pi(1)\}$.
			\item If $\epsilon < \norm{M_1 - M_2}_2$, the output satisfies: $\hat{k} = \pi(1)$. 
		\end{itemize} 
		Moreover, with probability at least $1-\delta$, its total number of queries, denoted $N$, satisfies:
		\[
		N \le \tilde c \times \log(1/\delta)H_{\epsilon},
		\] 
		where $\tilde c = \log^{2}(d)\log\left(H_{\epsilon}\right)\log(\log(H_{\epsilon}\vee d))$, where $H_{\epsilon} := \min \{H, d/\epsilon^2\}$ and where $c$ is a numerical constant. 
	\end{theorem}

	We now discuss some properties of the algorithms and its corresponding guarantees. The sample complexity of Algorithm~\ref{algo1} is of order $H=d/\norm{M_1 - M_2}_2^2$, matching the lower bound presented in Theorem~\ref{thm:lwrn} in Section~\ref{sec:low}  up to a poly-logarithmic factor. The dependence on the $L_2$ distance between experts is mainly due to the sampling scheme introduced in Algorithm~\ref{algo3}, relying on median elimination. We illustrate the gain of our algorithm with respect to a uniform allocation strategy across all tasks. Such sampling schemes rely on the average performances of each expert across all tasks as a criterion  (\citealp{heckel2019active, jamieson2015sparse, urvoy2013generic}). Suppose that the difference vector is $s$-sparse and the amplitude of the non-zero gaps is equal to $\Delta$. Then, sampling a task uniformly at random leads to an expected gap of $s\Delta/d$, which leads to a sample complexity of $\mathcal{O}(d^2/(s\Delta)^2)$ in order to identify the optimal expert. On the other hand, our bound suggests that the total number of queries scales as $\mathcal{O}(d/(s\Delta^2))$. This gain is due to the median elimination strategy allowing the concentration of the sampling effort on tasks with large gaps. More precisely, we prove that the initially sampled set of tasks $S_1$ contains a proportion of $s/d$ of non-zero gaps and that at each iteration in median elimination the last proportion increases by at least a constant factor $r>1$. Consequently, after roughly $\log(d/s)$ iteration, ``good tasks" constitute a constant proportion of the active tasks, which allows the algorithm to conclude by satisfying the stopping rule condition.

	\section{Ranking estimation (general $n$)}\label{sec:gen_e}
	
	\subsection{Ranking identification}
	
	In this section, we consider the task of ranking identification $\bm{(R)}$. Algorithm~\ref{algo4} takes as input $\delta \in (0,1)$ and outputs a ranking $\hat{\pi}$ of all the experts. We use Algorithm~\ref{algo1} with precision $\epsilon = 0$ to compare any pair experts. Given Algorithm~\ref{algo1} as a building block, we proceed using the binary insertion sort procedure: experts are inserted sequentially and the location of each expert is determined using a binary search. Algorithm~\ref{algo4} presents the procedure, where in each iteration (corresponding to the insertion of a new expert) a call to the procedure $\texttt{Binary-search}$ is made. A detailed implementation of the last procedure (using Algorithm~\ref{algo1})  is presented in Section~\ref{sec:full_ranking} in the appendix.   
	\begin{algorithm}[!ht] 
		\caption{\texttt{Active-ranking}$(\delta)$  \label{algo4} }
		\begin{algorithmic}
			\STATE \textbf{Input}: $\delta$.
			\STATE \textbf{Initialize}: Define $\hat{\pi}$ by $\hat{\pi}(1)=1$. 
			\FOR{$i=2,\dots,n$}
			\STATE $j \gets \texttt{Binary-search}(\tfrac{\delta}{n\lceil \log_2(n)\rceil}), i,\hat{\pi},1,i-1)$.
			\STATE Shift the rank of $\hat{\pi}[k]$ for $k \ge j$ by one.
			\STATE $\hat{\pi}[j] \gets i$
			\ENDFOR
			\STATE \textbf{Output:} $\hat{\pi}$. 
		\end{algorithmic}
	\end{algorithm}
	
	The following theorem presents the theoretical guarantees on the output and sample complexity of Algorithm~\ref{algo4}.

	\begin{theorem}\label{thm:upn}
		Suppose Assumption~\ref{assump} holds. Let $\delta \in (0,1)$, and define for $i \in \intr{n-1}$:
		\[
		H_i := \frac{d}{\norm{M_{\pi(i)} - M_{\pi(i+1)}}_2^2}.
		\]
		
		Let $H^* = \max_{i \le n-1} H_i$. With probability at least $1-\delta$, Algorithm~\ref{algo4} outputs the correct ranking, and its total number of queries  (denoted $N$) is upper bounded by:
		\[
		N \le \tilde{c} \times \log(1/\delta) \sum_{i=1}^{n-1} H_i\ ,
		\] 
		where $\tilde{c} = c\log(n)\log^2(d)\log(H^*)\log(n\log(H^*\vee d))$ and where $c$ is a numerical constant.
		
	\end{theorem}
	The first result states that Algorithm~\ref{algo4} with input $\delta$ is $\delta$-accurate. The second guarantee is a control on the total number of queries made by our procedure, with a large probability. As one would expect, the cost of full ranking of experts underlies the cost of distinguishing between two consecutive experts $\pi(i)$ and $\pi(i+1)$ for $i \in \intr{n-1}$. The last cost is characterized by the sample complexity $H_i = d/\norm{M_{\pi(i)} - M_{\pi(i+1)}}_2^2$. Consequently, the total number of queries made by Algorithm~\ref{algo4} is of order of $\mathcal{O}\left(\sum_{i=1}^{n-1} H_i\right)$ up to a poly-logarithmic factor.

	
	
	\subsection{Best expert identification:}

	Algorithm~\ref{algo7} takes as input the confidence parameter $\delta$ and outputs a candidate for the best expert. Note that while ranking the experts using Algorithm~\ref{algo4} would lead to identifying the top item, one would expect to use fewer queries for this relaxed objective. Algorithm~\ref{algo7} builds on a variant of the \texttt{Max-search} algorithm (a detailed implementation is presented in Section~\ref{sec:best_expert}). In the last algorithm, given a subset of expert $S$ and a precision $\epsilon$, we initially select an arbitrary element as a candidate for the best expert, then we perform comparisons with the remaining elements of $S$ using $\texttt{compare}(\delta/n, \epsilon)$ and  update the candidate for the best expert accordingly. This method leads to eliminating sub-optimal experts that have an $L_2$ distance with respect to expert $\pi(1)$ larger than $\mathcal{O}(\epsilon)$. Therefore, performing sequentially $\texttt{Max-search}$ and dividing the prescribed precision $\epsilon$ by $4$ after each iteration allows identifying the optimal expert after roughly $\log_4(1/\norm{M_{\pi(1)} - M_{\pi(2)}}_2^2)$ iterations.

	\begin{algorithm}[!ht] 
		\caption{\texttt{Best-expert}$(\delta)$  \label{algo7} }
		\begin{algorithmic}
			\STATE \textbf{Input}: $\delta$, $S = \intr{n}$, $\epsilon = 1$, $\delta'=\delta/2$
			\WHILE{$\abs{S}>1$}
			\STATE $S \gets \texttt{Max-search}(\delta', S, \epsilon)$.
			\STATE $\epsilon \gets \epsilon/4$
			\STATE $\delta' \gets \delta'/2$
			\ENDWHILE
			\STATE \textbf{Output:} $S$.		
		\end{algorithmic}
	\end{algorithm}
	
	The theorem below analyses the performance of Algorithm~\ref{algo7}.
	
	\begin{theorem}\label{thm:be}
		Suppose Assumption~\ref{assump} holds. Let $\delta \in (0,1)$, and define for $i \in \intr{2,n}$:
		\[
		G_i := \frac{d}{\norm{M_{\pi(1)} - M_{\pi(i)}}_2^2}.
		\]
		Let $G^* = \max_{i \in \intr{2,n}} G_i$. With probability at least $1-\delta$, Algorithm~\ref{algo7} outputs $\pi(1)$, and the total number of queries  (denoted $N$) 
		satisfies 
		\vspace{-3mm}
		\begin{equation*}
			N \le \tilde{c} \times \log(1/\delta) \sum_{i=2}^{n} G_i,
			\vspace{-3mm}
		\end{equation*} 
		with $\tilde{c} = c\log^{2}(d)\log^2\left(G^*\right) \log(n\log(d))$ and  $c$ a numerical constant.
		
	\end{theorem}
	
	In the first result, we show that Algorithm~\ref{algo7} is $\delta$-accurate for $\bm{(B)}$. The second result presents a bound on the total number of queries made by the algorithm. Observe that for each $i \in \{2, \dots,n\}$ the quantity $G_i = d/\norm{M_{\pi(1)} - M_{\pi(i)}}_2^2$ characterizes the sample complexity to distinguish expert $\pi(i)$ from the optimal expert. In order to identify the correct expert, a number of queries of order $G_i$ is required for each suboptimal expert $\pi(i)$, which leads to a total number of queries of order $\sum_{i=2}^{n} G_i$, up to a poly-logarithmic factor.

	\subsection{Discussion}
	
	Algorithms~\ref{algo4} and~\ref{algo7} are proven to be $\delta$ accurate for the ranking problem $\bm{(R)}$ and best expert $\bm{(B)}$. Moreover, Theorems~\ref{thm:lwrn_minimax} and~\ref{thm:lwrn_minimax_best_arm} in Section~\ref{sec:low} below show that their sample complexities are optimal up to a poly-logarithmic factor. While both procedures rely on comparing pairs of experts, their use of the $\texttt{compare}$ procedure presented in Algorithm~\ref{algo1} is different. The ranking procedure builds on $\texttt{compare}(\delta, 0)$, i.e., we need the exact order between compared experts. In contrast, for best expert identification, we only need approximate comparisons output by $\texttt{compare}(\delta, \epsilon)$ for an adequately chosen precision $\epsilon$. This difference is due to the fact that a complete ranking underlies comparing the closest pair of experts (in $L_2$ distance) while best expert identification requires distinguishing only the optimal expert from the sub-optimal ones.
	
	\section{Lower bounds}\label{sec:low}
	In this section, we provide some lower bounds on the number of queries of any $\delta$-accurate algorithm. As for the upper bound, we first consider the case of $2$ experts ($n=2$) in which ranking identification is the same as best expert identification. Then we consider the general case (any $n \geq 2$).
	
	\subsection{Lower bound in the case of two experts ($n=2$)}
	Problem-dependent lower bounds, i.e.~lower bounds that depend on the problem instance at hand\footnote{which is what we need in order to match the upper bound in Theorem~\ref{thm:up2}.}, are generally obtained by considering slight changes of any fixed problem, and proving that it is not possible to have an algorithm that performs well enough simultaneously on all the resulting problems. In the ranking problem, our monotonicity assumption constraint heavily restricts the nature of problem changes that we are allowed to consider.  
	
	For a given matrix $M_0\in \mathcal{M}_{2,d}$ representing the mean performance of a problem, a minimal and very natural class is the one containing $M_0 = ((M_0)_{1,.}, (M_0)_{2,.})^T$, and also the matrix $((M_0)_{2,.}, (M_0)_{1,.})^T$ where the two rows (experts) are permuted. In this class, however, we know the position of the question leading to maximal difference of performance between the two experts, as it is the question $j$ such that $|(M_{0})_{1,j} - (M_{0})_{2,j}|$ is maximised. So that an optimal strategy over this class would leverage this information by sampling only this question, and would be able to terminate using a number of query $N$ smaller in order than $\log(1/\delta)/\max_j\Big[\big((M_{0})_{1,j} - (M_{0})_{2,j}\big)^2\Big]$, with probability larger than $1-\delta$. Note that this is significantly smaller than our upper bound in Theorem~\ref{thm:up2}, but that the algorithm that we alluded to is dependent on the exact knowledge of the matrix $M_0$ - and in particular the positions of the informative questions - which is not available to the learner, and also very difficult to estimate. In order to include this absence of knowledge in the lower bound, we have to make the class of problems larger, by ensuring in particular that the position of informative questions are not available to the learner. A natural enlargement of the class of problem that takes this into account, but that is still very natural and tied to the matrix $M_0$, is to consider the class of all matrices whose gaps between experts are equal to those in $M_0$ up to a permutation of rows and columns (experts and questions). This is precisely the class that we consider in our lower bound, and that we detail below.
	
	

	


	Let $M_0\in \mathcal{M}_{2,d}$. Write $\underline \Delta = |(M_0)_{1,.}-(M_0)_{2,.}|$ for the vector of gaps between experts, and $\pi_0$ for the permutation that makes $M_0$ monotone - and we remind that $(M_0)_{\pi_0(2),.}$ is the least performant of the two experts.  Write $\mathbb{D}_{M_0}$ for the set of performance distributions, 
	namely of distributions of $X = (X_{i,j})_{i \in \intr{2}, j\in \intr{d}}$ corresponding to means $M$, such that: (i) Assumption~\ref{assump:sg} is satisfied for X, and (ii) The mean performance $M_{\pi}\in \mathcal{M}_{2,d}$, with associated permutation $\pi$ that transforms it into a monotone matrix, is such that $M_{\pi(2),.}= (M_0)_{\pi_0(2),.}$, and such that there exists a permutation $\sigma$ of $\intr{d}$ such that $M_{\pi(1),.}= (M_0)_{\pi_0(2),.} + \underline \Delta_{\sigma(.)}$.
	
	The set $\mathbb{D}_{M_0}$ is therefore the set of all distributions of $X$ that are $1$-sub-Gaussian, and where while one expert in $M$ is equal to the worst expert $(M_0)_{\pi_0(2),.}$ of $M_0$, the best expert is equal to $(M_0)_{\pi_0(2),.}$ plus a permutation of the gap vector $\underline \Delta$. This ensures that the gap structure over the mean performance is the same for all problems in $\mathbb{D}_{M_0}$.
	

The following theorem establishes a high probability lower bound on the termination time $N$ over the class of problems $\mathbb{D}_{M_0}$.
\begin{theorem}\label{thm:lwrn}
	Fix $n=2$. Let $d\geq 1$ and $\delta\in (0,1/4)$.  Consider any matrix $M_0\in \mathcal{M}_{2,d}$. For any $\delta$-accurate algorithm $A$ for either the ranking identification, or best expert identification (which is the same for $n=2)$, we have:
	\[
	\max_{\cB \in \mathbb{D}_{M_0}} \mathbb{P}_{\cB, A}\left[N \ge \frac{d}{2\|(M_0)_{1,.}-(M_0)_{2,.}\|_2^2} \log\left(\frac{1}{6\delta}\right)\right]\geq \delta,
	\]
	where $\mathbb{P}_{\cB, A}$ is the probability corresponding to samples collected by algorithm $A$ on problem $\cB$.
\end{theorem}
This theorem lower bounds the budget of any $\delta$-accurate algorithm by $\frac{d}{\|M_1-M_2\|_2^2} \log\Big(\frac{1}{6\delta}\Big)$, which matches up to logarithmic terms the upper bound in Theorem~\ref{thm:up2}. 
Interestingly, the query complexity depends therefore only on $\|(M_0)_1-(M_0)_2\|_2^2$, independently of the gap profile of $(M_0)_1-(M_0)_2$ - i.e.~of whether there are many small differences in performances across tasks, or a few large differences. Of course, a related optimal algorithm would solve differently the width versus depth trade-off on a sparse or dense problem, as discussed in Section~\ref{sec:two_e}, yet it does not show in the final bound on the query complexity thanks to the adaptivity of the sampling. A related phenomenon was already observed - albeit in a different regime and context - in~\citep{castro2014adaptive}.\\
The bound on $N$ in Theorem~\ref{thm:lwrn} is in high probability, on an event of high probability $1-\delta$ - where $1-\delta$ is also the minimal probability of being accurate for the algorithm. This matches our upper bound in Theorem~\ref{thm:up2}, where we also provide high probability upper bounds for $N$.







\subsection{Lower bound in the general case (any $n$)}

We now consider the general problem of ranking and best expert identification when $n>2$. As these two problems are not equivalent anymore, we provide two lower bounds.

In this part, we will consider classes of problems for constructing the lower bound which are wider than the one constructed for the case $n=2$, see Theorem~\ref{thm:lwrn} and the class $\mathbb D_{M_0}$. Driven by the fact that the quantity that appears there is the $L_2$ norm between experts, we will define the classes of problems by imposing constraints on the $L_2$ distance between pairs of experts.

\subsubsection{Ranking identification}
Fix any $\bm{\Delta} = (\Delta_1, \dots, \Delta_{n-1})$, such that $\Delta_i >0$ for each $i \le n-1$. Write $\mathbb{D}_{\bm{\Delta}}^{n,d}$ for the set of performance distributions, namely 
of distributions of $X = (X_{i,j})_{i \in \intr{n}, j\in \intr{d}}$ corresponding to means $M$, such that: (i) Assumption~\ref{assump:sg} is satisfied for X, and (ii) The mean performance $M\in\mathcal{M}_{n,d}$ satisfies
	\[
	\|M_{\pi(i)}-M_{\pi(i+1)}\|_2^2= \Delta^2_i,\text{ for }i=1,\ldots, n-1\ . 
	\]
The class $\mathbb{D}_{\bm{\Delta}}^{n,d}$ is such that the $L_2$ distance between the $i$-th best expert $\pi(i)$ and the $i+1$-th best expert $\pi(i+1)$ is fixed to $\Delta_i$. We however do not make further assumption on the gap structure within questions, as is done in Theorem~\ref{thm:lwrn} through the class $\mathbb D_{M_0}$.

Next, we provide a minimax lower bound for general $n$. The following theorem lower bound the expected budget when we fix the sequence of $L_2$ distances between consecutive rows. 
\begin{theorem}\label{thm:lwrn_minimax}
	Let $n, d \ge 1$, $\delta \in (0,1)$. For any $\delta$-accurate algorithm $A$ for ranking identification $\bm{(R)}$, we have:
	\[
	\max_{\cB \in \mathbb{D}^{n,d}_{\bm{\Delta}}} \mathbb{E}_{\cB, A}[N] \ge \sum_{i=1}^{n-1}\frac{d}{\Delta_i^2} \log(1/(4\delta)).
	\]
	where $\mathbb{E}_{\cB, A}$ is the probability corresponding to samples collected by algorithm $A$ on problem $\cB$. 
\end{theorem}


\subsubsection{Best expert identification}

Fix any positive and non-decreasing sequence $\overline{\bm{\Delta}} = (\overline{\Delta}_1, \dots, \overline{\Delta}_{n-1})$, such that $\Delta_i >0$ for each $i \le n-1$. Write $\overline{\mathbb{D}}_{\overline{\bm{\Delta}}}^{n,d}$ for the set of performance distributions, namely 
of distributions of $X = (X_{i,j})_{i \in \intr{n}, j\in \intr{d}}$ corresponding to means $M$, such that:

the set of distributions of experts performances such that: (i) Assumption~\ref{assump:sg} is satisfied for X, and (ii) The mean performances matrix $M\in\mathcal{M}_{n,d}$ satisfies 
\[
\|M_{\pi(1)}-M_{\pi(i+1)}\|_2^2\geq  \overline{\Delta}^2_i,\text{ for }i=1,\ldots, n-1\ . 
\]
The class $\overline{\mathbb{D}}_{\overline{\bm{\Delta}}}^{n,d}$ is such that the $L_2$ distance between the best expert $\pi(1)$ and the $i$-th best expert $\pi(i+1)$ is fixed to $\overline{\Delta}_i$. It is related to the construction for Theorem~\ref{thm:lwrn_minimax} of the set $\mathbb{D}_{\bm{\Delta}}^{n,d}$, yet here we only consider the distance to the best expert.

\begin{theorem}\label{thm:lwrn_minimax_best_arm}
	Let $n, d \ge 1$, $\delta \in (0,1)$. 
	For any $\delta$-accurate algorithm $A$ for best expert identification $\bm{(B)}$, we have:
	\[
	\min_{A \in \bA(\delta)} \max_{\cB \in \overline{\mathbb{D}}^{n,d}_{\bm{\Delta}}} \mathbb{E}_{\cB, A}[N] \ge \sum_{i=1}^{n-1}\frac{d}{\overline{\Delta}_i^2} \log(1/(4\delta)),
	\]
	where $\mathbb{E}_{\cB, A}$ is the probability corresponding to samples collected by algorithm $A$ on bandit problem $\cB$. 
\end{theorem}
\section{Numerical simulations}\label{sec:simu}
In this section, we perform some numerical simulations on synthetic data to compare our Algorithm with a benchmark procedure from the literature. We chose \texttt{AR} algorithm from \cite{heckel2019active} since they considered the problem of ranking experts in an active setting. Their method is based on pairwise comparisons and uses Borda scores as a criterion to rank experts. In order to harness their model into ours, we proceed as follows: when querying a pair of experts $(i,j)$, we sample a task uniformly at random from $\intr{d}$, and sample the performances of experts on this task then output the result.

We focus on the specific problem of identifying the best out of two experts ($n=2$) and $d=10$ tasks. For each $s \in \intr{d}$, we consider the following scenario:  the performances of both experts in each task follow a normal distribution with unit variance. The means of performances of the sub-optimal expert $M_{\pi(2)}$ are drawn from $[0, 1/2]$ following the uniform law. We build a gaps vector $\Delta_s$ that is $s$-sparse, the non-zero tasks are drawn uniformly at random from $\intr{d}$, and the value of the $k^{\text{th}}$ non-zero gap is set to $\left(\frac{k}{3s} \right)^2$. Then we consider $M_{\pi(1)} = M_{\pi(2)}+\Delta_s$. Figure\ref{fig1} presents the sample complexity of Algorithm~\ref{algo1} with parameters $(\delta, 0)$ and $\texttt{AR}$ from \cite{heckel2019active}, as a function of the sparsity ratio $s/d$ for $s\in \intr{d}$. The results are averaged over $20$ simulations for each scenario, in all simulations both algorithms made the correct output.

Figure~\ref{fig1} presents the empirical sample complexity of Algorithm~\ref{algo1} and $\texttt{AR}$ as a function of the sparsity rate $s/d$. The results show, as suggested by theoretical guarantees, that Algorithm~\ref{algo1} with parameters $(\delta, 0)$ outperforms \texttt{AR} for small sparsity rates, mainly due to its ability to detect large gaps, as discussed previously. For large sparsity rates, the gaps vector is dense, and evaluating experts using their average performance across all tasks proves to be efficient. In the last regime, \texttt{AR} procedure outperforms Algorithm~\ref{algo1}.   
\begin{figure}\label{fig1}
	\centering
	\includegraphics[width=8cm, height=5cm]{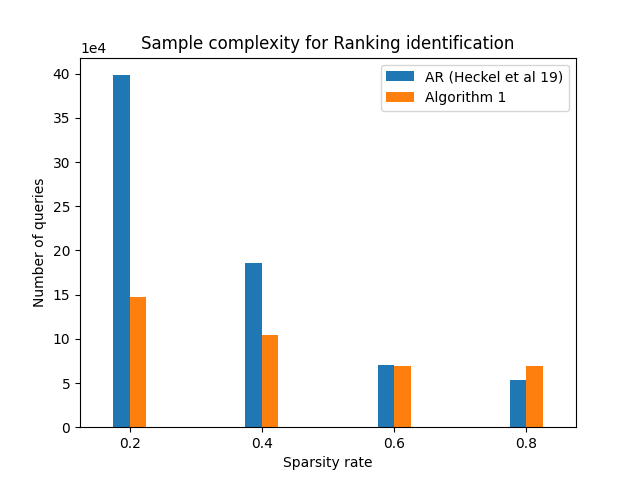}
	\caption{Empirical sample complexity of \texttt{AR} by \cite{heckel2019active} and Algorithm~\ref{algo1} in this paper. We varied the sparsity rate of the gaps vector $s/d \in [5\%, 100\%]$. The presented results are averaged over 20 simulations.}
\end{figure}

\section{Conclusion}

In this paper, we have addressed the challenge of ranking a set of $n$ experts based on sequential queries of their performance variables in $d$ tasks. By assuming the monotonicity of the mean performance matrix, we have introduced strategies that effectively determine the correct ranking of experts. These strategies optimize the allocation of queries to tasks with larger gaps between experts, resulting in a considerable improvement compared to traditional measures like Borda Scores.

Our research has unveiled several promising avenues for future exploration. One notable direction involves narrowing the poly-logarithmic gap in $n$ between our upper and lower bounds for both full ranking and best expert identification. Achieving this goal will require the development of more refined ranking strategies, which we leave for future investigation. Additionally, relaxing the monotonicity assumption considered in this study and adopting a more inclusive framework that accommodates diverse practical applications would be an intriguing area to explore. It would be worthwhile to scrutinize the assumptions made in the study conducted by \cite{bengs2021preference} as a potential direction for further research.
\vspace{-0.2cm}
\paragraph{Acknowledgments:} The work of A. Carpentier is partially supported by the Deutsche Forschungsgemeinschaft
(DFG) Emmy Noether grant MuSyAD (CA 1488/1-1), by the DFG – 314838170, GRK 2297
MathCoRe, by the FG DFG, by the DFG CRC 1294 ‘Data Assimilation’, Project A03, by the
Forschungsgruppe FOR 5381 “Mathematical Statistics in the Information Age – Statistical
Efficiency and Computational Tractability”, Project TP 02, by the Agence Nationale de la
Recherche (ANR) and the DFG on the French-German PRCI ANR ASCAI CA 1488/4-1
“Aktive und Batch-Segmentierung, Clustering und Seriation: Grundlagen der KI” and by
the UFA-DFH through the French-German Doktorandenkolleg CDFA 01-18 and by the SFI
Sachsen-Anhalt for the project RE-BCI. The work of E.M. Saad and N. Verzelen 
is supported by ANR-21-CE23-0035 (ASCAI).


\bibliography{biblio}
\bibliographystyle{plain}

\newpage
\appendix
\onecolumn

\section{Additional numerical simulations}

We add a numerical simulation on synthetically generated data. We consider the task of comapring two experts and suppose that the performance gaps vector is sparse, with a fixed sparsity rate of $1/3$. We conduct simulations for various dimension (number of tasks $d$) and plot the sample complexities of our algorithm and the benchmark algorithm AR (\cite{heckel2019active}). Figure~\ref{sim2} displays the results. 
\begin{figure}\label{sim2}
	\centering
	\includegraphics[width=8cm, height=5cm]{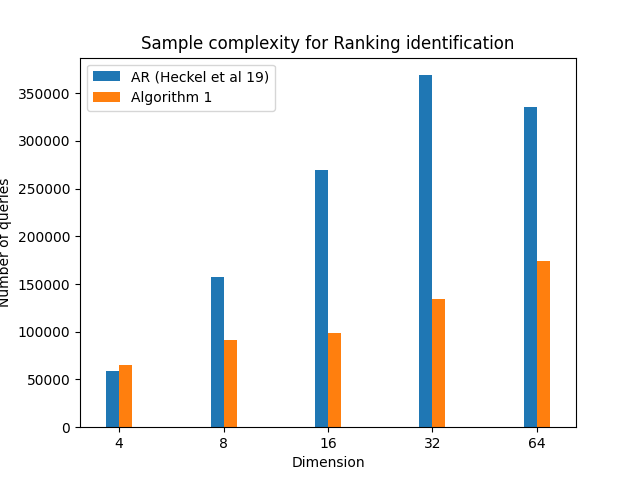}
	\caption{Empirical sample complexity of \texttt{AR} by \cite{heckel2019active} and Algorithm~\ref{algo1} in this paper. We varied the number of tasks $d \in [4, 8, 16, 32, 64]$ while keeping the sparsity rate constant $1/3$. The presented results are averaged over 20 simulations.}
\end{figure}

\section{Proof of Theorem~\ref{thm:up2}}
Suppose Assumption~\ref{assump} holds and that (w.l.o.g) expert 1 is the optimal expert, that is $\pi(1)=1$.
\paragraph{Additional notation:} Define $x_j := M_{1,j} - M_{2,j}$. Let $(x_{(j)})_{j \in \intr{d}}$ denotes the $(x_{j})_{j \in \intr{d}}$  ordered in a decreasing order. Next, we introduce the effective sparsity $s^*$ of the vector $x$. Lemma~\ref{lem:sparse} states that there exists $s^*\in \intr{d}$ such that
\[
\|x\|_2^2 \le s^*\log(2d)~x_{(s^*)}^2\ .
\]
If $x$ had been $s$-sparse and if all the non-zero components of $x$ had been equal, then we would have $s=s^*$ and the above inequality would be valid (without the $\log(2d)$). Here, the virtue of the effective sparsity $s^*$ is that it is defined for arbitrary vectors $x$. Then, we denote
\begin{equation*}
	x_* := x_{(s^*)} \ , \qquad  \Delta_*^2 := s^* x^2_*\ ,\qquad\text{ and } \quad \cS^* := \{ i \in \intr{d}: x_i \ge x_*\}\ ,
\end{equation*}
where $\Delta_*^2$ plays the role of the square norm $x$ at the scale $s^*$ and $\cS^*$ is the corresponding set of size $s^*$ of coordinates larger or equal to $x_*$.

\noindent In order to prove Theorem~\ref{thm:up2}, we will follow three steps:
\begin{itemize}
	\item Part 1: In Lemma~\ref{lem:finpart1}, we show that Algorithm~\ref{algo1} outputs $\hat{k} \in \{\texttt{null}, 1 \}$ with probability at least $1-\delta$.
	\item Part 2: In Lemma~\ref{lem:finpart2} we consider Algorithm~\ref{algo3} with input $(\delta, s, h)$. If $h < x_*$ and $s <s^*$, then, with probability at least $1-\log(d)\delta$, the output of the algorithm is $\hat{k}= 1$.
	\item Part 3: As a conclusion, we will gather previous lemmas to derive a bound on the total number of queries until the termination of Algorithm~\ref{algo1}. 
\end{itemize}

.

\paragraph{Part 1:} Let us start by introducing the following concentration result for the empirical means computed in Algorithm~\ref{algo3}. In the lemma below,  the set $S$ is considered as fixed. More generally,  $S$ can be any random set independent of the samples used to compute the means $(\hat{\mu}_{j,T})_{j,T}$

\begin{lemma}\label{lem:hoef}
	Let $T \ge 1$, fix $j \in \intr{d}$ and $S \subset \intr{d}$. We have
	\begin{equation*}
		\mathbb{P}\left( \frac{1}{\abs{S}}\sum_{j \in S} \hat{\mu}^{(21)}_{j,T} > \sqrt{\frac{2\log(1/\delta)}{\abs{S}T}} \right) \le \delta,
	\end{equation*}
	where $\hat{\mu}^{(21)}_{j,T} = \frac{1}{T} \sum_{u=1}^{T} \left(X_{2,j}^{(u)} - X_{1,j}^{(u)}\right)$.
\end{lemma}
\begin{proof} The proof is a straightforward consequence of Chernoff's inequality. Recall that the variable $X_{2,j} - X_{1,j}$ is $\sqrt{2}$-sub-Gaussian for any $j \in S$. Moreover, all samples used in the sum $\sum_{j \in S} \hat{\mu}_{j,T}^{(21)}$ are independent. We have
	\begin{align*}
		\mathbb{P}\left(\frac{1}{\abs{S}} \sum_{j\in S}\hat{\mu}^{(21)}_{j,T} > \sqrt{\frac{2\log(1/\delta)}{\abs{S}T}} \right) &\le \mathbb{P}\left( \frac{1}{\abs{S}} \sum_{j\in S}\hat{\mu}^{(21)}_{j,T} + \frac{1}{\abs{S}} \sum_{j\in S}x_j > \sqrt{\frac{2\log(1/\delta)}{\abs{S}T}} \right) \le \delta,
	\end{align*}
	where the first inequality follows from the fact $x_j \ge 0$ (recall that the optimal expert is $1$), and the second is a direct consequence of Chernoff's concentration inequality.
	
\end{proof}

\begin{lemma}\label{lem:delta1}
	Consider Algorithm~\ref{algo3} with input $(\delta, s, h)$. We have
	\[
	\mathbb{P}\left( \hat{k}=2\right) \le 1.75\log\left(d\right)\delta.
	\]
	
\end{lemma}
\begin{proof}
	We have
	\begin{align*}
		\mathbb{P}\left(\hat{k}= 2\right) &= \mathbb{P}\left( \exists \ell \le \log_{4/3}(d/s):  \hat{\mu}_{\ell}^{(21)} > \sqrt{\frac{2\log(2/\delta)}{n_0\phi}}\right)
		\le \log_{4/3}(d/s)\frac{\delta}{2}\leq 1.75\log(d)\delta\  ,
	\end{align*}
	where we used Lemma~\ref{lem:hoef}.
\end{proof}

\begin{lemma}\label{lem:finpart1}
	Consider Algorithm~\ref{algo1} with input $(\delta, \epsilon)$. The probability of outputting the wrong result satisfies
	\[
	\mathbb{P}\left( \hat{k}= 2\right) \le 0.6\delta\ .
	\]
\end{lemma}
\begin{proof}
	When Algorithm~\ref{algo1} with input $\delta$ halts, we denote $\hat{k}$ its returned ranking. For $\rho \ge 1$ and $r\le \rho$, denote $\hat{k}_{\rho,r}$ the output of Algorithm~\ref{algo3} with input $(\delta_{\rho}, s_r, h_r)$. We have
	\begin{align*}
		\mathbb{P}\left(\hat{k}= 2\right) &= \mathbb{P}\left( \exists \rho \ge 1, \exists r \le \rho: \hat{k}_{\rho,r} = 2  \right)	\le \sum_{\rho=1}^{\infty} \sum_{r=0}^{\rho} \mathbb{P}\left( \hat{k}_{\rho,r} = 2 \right)\ .
	\end{align*}
	Using Lemma~\ref{lem:delta1}, we have
	\begin{align*}
		\mathbb{P}\left(\hat{k}= 2\right) &\le \sum_{\rho=1}^{\infty} \sum_{r=0}^{\rho} 1.75\log(d) \frac{\delta}{10\log(d)\rho^3} \le \frac{1.75\delta}{10} \sum_{\rho=1}^{\infty} \frac{\rho+1}{\rho^3}\le 0.6\delta\ .
	\end{align*}
\end{proof}

\paragraph{Part 2:} To ease notation, we write $S_{\ell}$ instead of $S^{(12)}_{\ell}$ in the remainder of this proof. Introduce the following notation
\begin{align*}
	\cT^* &:= \{j \in \intr{d}: x_{j}^{(12)} \ge \frac{1}{2} x_*  \}\ ; \quad \quad 
	M_{\ell} := \abs{\{ j \in S_{\ell} \cap \cS^*: \hat{\mu}_{j,\ell}^{(12)} \ge \frac{3}{4} x_*\}}\ ;\\
	N_{\ell} &:= \abs{\{ j \in S_{\ell} \setminus \cT^*:  \hat{\mu}_{j,\ell}^{(12)} \ge \frac{3}{4} x_* \}}\ ;\quad \quad 
	q_{\ell} := \exp(-2^{\ell}) ;\\
	\xi_{\ell} &:= \left\lbrace \frac{\abs{S_{\ell}\cap \cS^*}}{\abs{S_{\ell}}} \ge \frac{s^*}{2d} \left(\frac{4}{3}\right)^{\ell-1}  \right\rbrace\ ; \quad \quad 
	\xi'_{\ell+1} := \left\lbrace \frac{\abs{S_{\ell+1}\cap \cS^*}}{\abs{S_{\ell+1}}} \ge \frac{4}{3} \frac{\abs{S_{\ell}\cap \cS^*}}{\abs{S_{\ell}}} \right\rbrace\ ;\\
	\Theta_{\ell} &:= \left\lbrace \frac{\abs{S_{\ell}\cap (\cT^* \setminus \cS^*) }}{\abs{S_{\ell}}} < \frac{1}{4}  \right\rbrace\ ; \quad \quad 
	\Lambda_{\ell} := \left\lbrace \frac{\abs{S_{\ell}\cap  \cS^* }}{\abs{S_{\ell}}} < \frac{1}{4}  \right\rbrace\ .
\end{align*}
While the set $\cS^*$ (introduced earlier) stands for the collection of coordinates that are larger or equal to $x_*$, the set $\cT^*$ contains all moderate coordinates. The other new notation deal with the constant with the $\ell$-th iteration in the algorithm. $M_\ell$ corresponds to the set of 'significant' coordinates that lie in $S_\ell$ and such that the empirical mean is large enough, while $N_\ell$ is the set of 'small' coordinates that are in $S_\ell$ and whose empirical mean is large. The events $\xi_{\ell}$, $\xi'_{\ell}$, $\Theta_\ell$, and $\Lambda_\ell$ are discussed later.
The following lemma states the set $S_1$ at step $\ell=1$ contains a non-vanishing  proportion of significant coordinates. 

\begin{lemma}\label{lem:3}
Consider Algorithm~\ref{algo3} with input $(\delta, s, h)$. Suppose that $s<s^*$. Recall that $S_1$ is the set of sampled questions. We have
\[
\mathbb{P}\left[ \abs{S_1 \cap \cS^*} \ge \frac{s_*}{2d}|S_1| \right] \ge 1-\delta\ .
\]
\end{lemma}
\begin{proof}
By definition, $\abs{S_1\cap S^*}$  follows a binomial distribution with parameters $(\abs{S_1}, s^*/d)$. Hence, Chernoff's inequality (Lemma~\ref{lem:cher} with $\kappa = 1/2$) enforces that 
\begin{align*}
	\mathbb{P}\left(\abs{S_1 \cap \cS^*} \le \frac{s^*}{2d} \abs{S_1} \right) &\le \exp\left(- \frac{s^*}{8d}  \abs{S_1}\right)
	\le \exp\left( - \frac{s^*}{s} \log(1/\delta)\right) \le \delta\ ,
\end{align*} 
where we used $\abs{S_1}=\phi \geq 26 \log(1/\delta)d/s$ and $s<s^*$. 
\end{proof}

The next lemma roughly states that, provided the event $\xi_\ell$ holds, then, with high probability, $S_{l+1}$ will contain a larger proportion of significant questions. More precisely, we prove that the number  $M_\ell$  of significant questions with a large empirical mean is high and the number $S_\ell$ of non-significant questions with a large empirical mean is small. 
\begin{lemma}\label{lem:1}
Consider Algorithm~\ref{algo3} with inputs $(\delta, s, h)$ such that $s<s^*$ and $h < x_*$. For any $\ell\geq 1$, we have 
\begin{align*}
	\mathbb{P}\left( M_{\ell} \le \frac{2}{3} \abs{S_{\ell} \cap \cS^*} \Big|\xi_{\ell}\right) &\le \delta\ ;\quad \quad 
	\mathbb{P}\left( N_{\ell} \ge \frac{1}{4} \abs{S_{\ell}} \Big| \xi_{\ell}\right)  \le \delta\ .
\end{align*}
\end{lemma}
\begin{proof}
Define $\bar{M}_{\ell} := \abs{S_{\ell}\cap \cS^*}$. Using the definition of $M_{\ell}$, we have
\begin{equation*}
	\bar{M}_{\ell} = \sum_{j \in S_{\ell} \cap \cS^*} \mathds{1}\left( \hat{\mu}_{j,\ell}^{(12)} < \frac{3}{4} x_* \right).
\end{equation*}
Let $j \in S_{\ell} \cap \cS^*$, recall that $\mathbb{E}[\hat{\mu}_{j,\ell}^{(12)}] \ge x_*$, and the samples $(X_{1,j}^{(u)} - X_{2,j}^{(u)})$ are $\sqrt{2}$-sub-Gaussian. Therefore, using Chernoff's bound
\begin{align*}
	\mathbb{P}\left( \hat{\mu}_{j,\ell}^{(12)} < \frac{3}{4} x_* \right) &\le \exp\left(-t_{\ell}\frac{x_*^2}{32}  \right)
	= \exp\left(-2^{\ell-1}n_0\frac{x_*^2}{32}\right)\leq \exp(-2^{\ell})= q_{\ell}\ ,
\end{align*}
since $h\leq x_*$ and by definition of $\ell$ and of $q_{\ell}$. 
Thus, the variables $\mathds{1}( \hat{\mu}_{j,\ell}^{(12)} < \frac{3}{4} \Delta^* )$ for $j \in S_{\ell} \cap \cS^*$ are stochastically dominated by $\cB(q_{\ell})$ and independent. 
Therefore $\bar{M}_{\ell}$ is stochastically dominated by $\cB\left(\abs{S_{\ell} \cap \cS^*}, q_{\ell}\right)$, where $\cB(a,b)$ denotes the binomial distribution with parameters $(a,b)$.
Let $\kappa :=  1/(3q_{\ell}) -1 >0 $, we have
\begin{align*}
	\mathbb{P}\left( M_{\ell} \le \frac{2}{3} \abs{S_{\ell} \cap \cS^*} \Big| \xi_{\ell}\right) &= 
	\mathbb{P}\left( \bar{M}_{\ell} > \frac{1}{3} \abs{S_{\ell} \cap \cS^*}  \Big| \xi_{\ell}\right)\\
	&= \mathbb{P}\left(\bar{M}_{\ell} > (1+\kappa)q_{\ell} \abs{S_{\ell} \cap \cS^*} \Big| \xi_{\ell} \right)
\end{align*}
Using Lemma~\ref{lem:cher} and the definition of the event $\xi_{\ell}$
\begin{align*}	
	\mathbb{P}\left( M_{\ell} \le \frac{2}{3} \abs{S_{\ell} \cap \cS^*} \Big| \xi_{\ell}\right) &\le
	\exp\left\lbrace \kappa q_{\ell} \abs{S_{\ell} \cap \cS^*} - (1+\kappa)q_{\ell} \abs{S_{\ell}\cap \cS^*} \log(1+\kappa)  \right\rbrace\\
	&\le \exp\left\lbrace \left(\frac{1}{3} - q_{\ell} - \frac{1}{3} \log\left(\frac{1}{3q_{\ell}}\right) \right) \abs{S_{\ell} \cap \cS^*}\right\rbrace\ .
\end{align*}
Recall that $q_{\ell} = \exp(-2^{\ell})$ so that $\kappa > 1/5$. Using the expression $\abs{S_{\ell}} = \frac{\abs{S_1}}{2^{\ell-1}} \ge \frac{26d\log(1/\delta)}{2^{\ell-1}s}$, the definition of the event $\xi_{\ell}$, and $s>s^*$, we deduce that 
\begin{align*}
	\mathbb{P}\left( M_{\ell} \le \frac{2}{3} \abs{S_{\ell} \cap \cS^*}  \Big|\xi_{\ell}\right) &\le 
	\exp\left\lbrace \left(\frac{1}{3} - \exp(-2^{\ell}) - \frac{2^{l}-\log(3)}{3}   \right) \left(\frac{2}{3}\right)^{\ell -1}13\log(1/\delta) \right\rbrace .
\end{align*}
For $\ell\geq 3$, we have $2^{\ell-1}\geq \log(3)+3$, which leads to 
\begin{align*}
	\mathbb{P}\left( M_{\ell} \le \frac{2}{3} \abs{S_{\ell} \cap \cS^*}  \Big|\xi_{\ell}\right) &\le 
	\exp\left\lbrace -\left(\frac{4}{3}\right)^{\ell -1}13\log(1/\delta) \right\rbrace \leq \delta\ .
\end{align*}
One easily check that this bound is still true for $\ell=1$, and $\ell=2$. We have proved the first part of the lemma.

\noindent For the second result, we start from
\begin{equation*}
	N_{\ell} = \sum_{j \in S_{\ell} \setminus \cT^*}\mathds{1}\left( \hat{\mu}_{j,\ell}^{(12)} \ge \frac{3}{4} x_* \right).
\end{equation*}
Arguing as in the first part, we easily check that the variables $\mathds{1}(\hat{\mu}_{j,\ell}^{(12)} \ge \frac{3}{4} x_* )$, for $j \in S_{\ell} \setminus \cT^*$ are independent and stochastically dominated by $\cB(q_{\ell})$. Hence $N_{\ell}$ is stochastically dominated by $\cB(\abs{S_{\ell}}, q_{\ell})$. Let $\kappa' = \frac{1}{4q_{\ell}}-1 \ge 0$. Using again Lemma~\ref{lem:cher}, we have 
\begin{align*}
	\mathbb{P}\left( N_{\ell}\ge \frac{1}{4} \abs{S_{\ell} } \Big| \xi_{\ell} \right) &= \mathbb{P}\left(  N_{\ell}\ge (1+\kappa)q_{\ell} \abs{S_{\ell}} \Big| \xi_{\ell} \right)\\ 
	&\le \exp\left\lbrace \kappa q_{\ell} \abs{S_{\ell}} - (1+\kappa) q_{\ell}\abs{S_{\ell}} \log(1+\kappa) \right\rbrace \\
	&\le \exp\left\lbrace \left(\frac{1}{4} - q_{\ell}  - \frac{1}{4} \log\left(\frac{1}{4q_{\ell}}\right)\right) \abs{S_{\ell}}  \right\rbrace.
\end{align*}
Next, we use the expression of $\abs{S_{\ell}}$ and obtain
\begin{align*}
	\mathbb{P}\left( N_{\ell} \ge \frac{1}{4} \abs{S_{\ell} } \Big| \xi_{\ell}\right) &\le \exp\left\lbrace   \left(\frac{1}{4} - \exp(-2^{\ell})  - \frac{1}{4} \log\left(\frac{\exp(2^{\ell})}{4}\right)\right) \frac{26d\log(1/\delta)}{s 2^{\ell-1}}  \right\rbrace\ . 
\end{align*}
Recall that $d\geq s$. As in the first part of the proof, we easily check that $\ell=1$ and $\ell=2$ the above expression is smaller or equal to $\delta$. For $\ell\geq 3$, we have $2^{l-1}\geq 1+\log(4)$, which implies also that 
\[
\mathbb{P}\left( N_{\ell}\ge \frac{1}{4} \abs{S_{\ell} } \Big| \xi_{\ell} \right)\leq \exp\left[- \frac{2^{\ell-1}}{4}\cdot\frac{26}{2^{\ell-1}}\log(1/\delta)\right]\leq \delta\ . 
\]
\end{proof}

The event $\Lambda_{\ell}$ and $\Theta_{\ell}$ respectively state that the proportion of  significant (and moderately significant) questions in $|S_{\ell}|$ is smaller or equal to $1/4$. The following lemma roughly states that as long $\xi_{\ell}$, $\Theta_{\ell+1}$, and $\Lambda_{\ell+1}$, then the proportion of significant questions in $S_{\ell+1}$ is significantly reduced.

\begin{lemma}\label{lem:2}
Let any integer $\ell\geq 1$, we have 
\[
\mathbb{P}\left( \Theta_{\ell+1} \text{ and }  \Lambda_{\ell+1}  \text{ and }  \neg{\xi'_{\ell+1}}  \Big| \xi_{\ell}  \right) \le 2 \delta\enspace .
\]
\end{lemma}
\begin{proof}
Recall that $m^{(12)}_{\ell}$ denotes the median computed in Algorithm~\ref{algo3}, we will denote $m_{\ell}$ instead of $m^{(12)}_{\ell}$ in this proof for the sake of simplicity.  We have
\begin{align}\label{eq:lem1}
	\mathbb{P}\left( \Theta_{\ell+1} \text{ and } \Lambda_{\ell+1} \text{ and } \neg{\xi'_{\ell+1}}\Big| \xi_{\ell} \right) &= \mathbb{P}\left( \Theta_{\ell+1} \text{ and } \Lambda_{\ell+1} \text{ and }  \neg{\xi'_{\ell+1}} \text{ and } \{m_{\ell} < \frac{3}{4} x_*\} \Big| \xi_{\ell}\right) \nonumber\\
	&\qquad  + \mathbb{P}\left( \Theta_{\ell+1} \text{ and } \Lambda_{\ell+1} \text{ and }  \neg{\xi'_{\ell+1}} \text{ and } \{m_{\ell} \ge \frac{3}{4} x_*\} \Big| \xi_{\ell}\right).
\end{align}
\paragraph{Upper bound for the first term in the rhs of \eqref{eq:lem1}.} The event $m_{\ell} <(3/4) x_*$ implies that
$\lbrace j \in S_{\ell} \cap \cS^*: \hat{\mu}_{j,\ell}^{(12)} \ge \frac{3}{4} x_* \rbrace \subset S_{\ell+1}$. Hence, 
\[
\left\lbrace j \in S_{\ell} \cap \cS^*: \hat{\mu}_{j,\ell}^{(12)} \ge \frac{3}{4} x_* \right\rbrace \subset S_{\ell+1} \cap \cS^*,
\]
which gives $M_{\ell} \le \abs{S_{\ell+1} \cap \cS^*}$. This leads us to
\begin{align*}
	\mathbb{P}\left( \Theta_{\ell+1} \text{ and } \Lambda_{\ell+1} \text{ and } \neg{\xi'_{\ell+1}} \text{ and } \{m_{\ell} < \frac{3}{4} x_*\} \Big| \xi_{\ell} \right) &\le \mathbb{P}\left(\neg{\xi'_{\ell+1}} \text{ and } \{m_{\ell} < \frac{3}{4} x_*\} \Big| \xi_{\ell}\right)\\
	&= \mathbb{P}\left( \{\abs{S_{\ell+1} \cap \cS^*} \le \frac{2}{3} \abs{S_{\ell} \cap \cS^*}\} \text{ and } \{m_{\ell} < \frac{3}{4} x_*\}\Big| \xi_{\ell} \right)\\
	&\le \mathbb{P}\left(  \{ M_{\ell} <\frac{2}{3} \abs{S_{\ell} \cap \cS^*} \} \text{ and } \{m_{\ell} < \frac{3}{4} x_*\} \Big| \xi_{\ell}\right)\\
	&\le \mathbb{P}\left(   M_{\ell} <\frac{2}{3} \abs{S_{\ell} \cap \cS^*} \Big| \xi_{\ell}\right) \le \delta\ ,
\end{align*}
where we used Lemma~\ref{lem:1} in the last line.
\paragraph{Upper bound for the second term in the rhs of \eqref{eq:lem1}.}
The event $\Theta_{\ell+1}$ implies in particular that
\[
\abs{S_{\ell+1} \cap \cS^*} + \abs{S_{\ell+1} \setminus \cT^*} \ge \frac{3}{4} \abs{S_{\ell+1}}.
\]
Moreover, we have that the event $\{ m_{\ell} \ge \frac{3}{4} x_*\}$ implies that
\begin{equation*}
	\left\lbrace j \in S_{\ell} \setminus \cT^*: \hat{\mu}_{j,\ell}^{(12)} \ge \frac{3}{4} x_* \right\rbrace \supset S_{\ell+1}\setminus  \cT^*,
\end{equation*}
which gives $\abs{S_{\ell+1} \cap \cS^*}+N_{\ell} \ge \abs{S_{\ell+1} \cap \cS^*} + \abs{S_{\ell+1} \setminus \cT^*} \ge \frac{3}{4} \abs{S_{\ell+1}}$.
We conclude that 
\begin{align*}
	\mathbb{P}\left( \Theta_{\ell+1} \text{ and } \Lambda_{\ell+1} \text{ and } \neg{\xi'_{\ell+1}} \text{ and } \{m_{\ell} \ge \frac{3}{4} x_*\}\Big| \xi_{\ell}\right) &\le \mathbb{P}\left(\Lambda_{\ell+1} \text{ and } \{\abs{S_{\ell+1} \cap \cS^*} + N_{\ell} \ge \frac{3}{4} \abs{S_{\ell+1}} \}\Big| \xi_{\ell}\right)\\
	&\le \mathbb{P}\left( N_{\ell} \ge \frac{1}{2} \abs{S_{\ell+1}} \Big| \xi_{\ell}\right)= \mathbb{P}\left( N_{\ell} \ge \frac{1}{4} \abs{S_{\ell}} \Big| \xi_{\ell}\right) \le \delta,
\end{align*}
where we used the definition of $\Lambda_{\ell+1}$ and Lemma~\ref{lem:1} in the second line.
We conclude that 
\[
\mathbb{P}\left( \Theta_{\ell+1} \text{ and } \Lambda_{\ell+1} \text{ and } \neg{\xi'_{\ell+1}} \Big| \xi_{\ell}\right) \le 2\delta\ .
\]

\end{proof}

\noindent Consider Algorithm~\ref{algo3} with input $(\delta, s, h)$. 
Let $\bar{\ell} := \lceil \log_{4/3}(d/s^*)\rceil+1$. Introduce the following event:
\begin{equation}
E_{\ell} := \{\text{Algorithm 2 made at least } \ell \text{ iterations}\} \text{ and }\bigcap_{u=1}^{u = \ell} \left(\Theta_{u} \cap \Lambda_{u}\right).
\end{equation}

\begin{lemma}\label{lemE}
Consider Algorithm~\ref{algo3} with input $(\delta, s, h)$. Suppose that $h < x_*$ and $s<s^*$. We have
\[
\mathbb{P}\left(E_{\bar{\ell}}  \right) \le 2\bar{\ell}\delta.
\]
\end{lemma}
\begin{proof}
First, let us prove by induction that, for any $\ell \le \bar{\ell}$, we have 
\[
\mathbb{P}\left( \neg{\xi_{\ell}} \text{ and } E_{\ell} \right) \le 2\ell\delta.
\]
The assertion for $\ell = 1$ is a direct consequence of Lemma~\ref{lem:3}. Suppose that the result for some $\ell \leq \bar{\ell}$. Observe that, by definition, the event $\neg{\xi_{\ell+1}}$ implies that either we have $\neg{\xi'_{\ell+1}}$ or we have $\neg{\xi_{\ell}}$. This leads us to 
\begin{align*}
	\mathbb{P}\left(\neg{\xi_{\ell+1}} \text{ and } E_{\ell+1} \right) &\le \mathbb{P}\left(\neg{\xi'_{\ell+1}} \text{ and } E_{\ell+1} \right) + \mathbb{P}\left(\neg{\xi_{\ell}} \text{ and } E_{\ell+1} \right)\\
	&\le \mathbb{P}\left(\neg{\xi'_{\ell+1}} \text{ and } \Lambda_{\ell+1} \text{ and } \Theta_{\ell+1} \right) + \mathbb{P}\left(\neg{\xi_{\ell}} \text{ and } E_{\ell} \right)\\
	&\le 2\delta+ 2\ell \delta,
\end{align*}
where we used Lemma~\ref{lem:2} and induction hypothesis in the last line. 

Observe that
\begin{align*}
	\neg \xi_{\bar{\ell}} &= \left\lbrace \frac{\abs{S_{\bar{\ell}} \cap \cS^*}}{\abs{S_{\bar{\ell}}}} < \frac{s^{*}}{2d} \left(\frac{4}{3}\right)^{\bar{\ell-1}} \right\rbrace \ ; \quad \quad 
	\Lambda_{\bar{\ell}} = \left\lbrace \frac{\abs{S_{\bar{\ell}} \cap \cS^*}}{\abs{S_{\bar{\ell}}}} < \frac{1}{4} \right\rbrace.
\end{align*} 
Recall that $\bar{\ell} = \lceil \log_{4/3}(d/s^*)\rceil +1$, which implies that $\frac{s^{*}}{2d} \left(\frac{4}{3}\right)^{\bar{\ell}} \geq  \frac{1}{4}$. Hence, we have $\neg \xi_{\bar{\ell}}\subset \Lambda_{\bar{\ell}}$ and we conclude 
that $\mathbb{P}(E_{\bar{\ell}})= \mathbb{P}(\neg \xi_{\bar{\ell}} \text{ and } E_{\bar{\ell}}) \leq 2\ell \delta$.
%
%
%
\end{proof}

The following result is built upon the previous lemmas and states that Algorithm~\ref{algo3} returns \texttt{null} with a very small probability provided that $h$ and $s$ are small enough. 
\begin{lemma}\label{lem:finpart2}
Consider Algorithm~\ref{algo3} with input $(\delta, s, h)$. Suppose that $h < x_*$ and $s<s^*$. We have
\[
\mathbb{P}\left( \{\hat{k}= \texttt{null}\}   \right) \le 3\bar{\ell}\delta.
\]
\end{lemma}
\begin{proof}
We have:
\begin{align*}
	\mathbb{P}\left( \{\hat{k}= \texttt{null}\}\right) &= \mathbb{P}\left( \{\hat{k}= \texttt{null}\} \text{ and }  \neg{E_{\bar{\ell}}}  \right) + \mathbb{P}\left( \{\hat{k}= \texttt{null}\} \text{ and }  E_{\bar{\ell}}  \right)\\
	&\le \mathbb{P}\left( \{\hat{k}= \texttt{null}\} \text{ and }  \neg{E_{\bar{\ell}}}  \right) + 2\bar{\ell}\delta,
\end{align*}
where we used Lemma~\ref{lemE}.	Now Suppose $E_{\bar{\ell}}$ is false. Hence,  there exists an iteration $\ell \le \bar{\ell}$ such that either $\Theta_{\ell}$ is false or $\Lambda_{\ell}$ is false. Recall that $\{\hat{k}= \texttt{null} \}$ implies that $\hat{\mu}_{\ell}^{(12)} < \sqrt{\frac{2\log(2/\delta)}{n_0\phi}}$ (otherwise, the algorithm halts at iteration $\ell$ and outputs $\hat{k} \in \{1,2\}$).
Then
\begin{align*}
	\mathbb{P}\left(  \{\hat{k}= \texttt{null}  \} \text{ and }  \neg{E}  \right) &\le \sum_{\ell=1}^{\bar{\ell}}\mathbb{P}\left( \left\lbrace \hat{\mu}_{\ell}^{(12)} <  \sqrt{\frac{2\log(2/\delta)}{n_0\phi}}\right\rbrace \text{ and } \{\neg{\Theta_{\ell}} \cup \neg{\Lambda_{\ell}} \}\right)\\
	&= \sum_{\ell=1}^{\bar{\ell}}\mathbb{P}\left( \left\lbrace\frac{1}{\abs{S_{\ell}}}\sum_{j\in S_{\ell}}x_{j} -\hat{\mu}_{\ell}^{(12)}  > \frac{1}{\abs{S_{\ell}}}\sum_{j\in S_{\ell}}x_{j} - \sqrt{\frac{2\log(2/\delta)}{n_0\phi}} \right\rbrace \text{ and } \{\neg{\Theta_{\ell}} \cup \neg{\Lambda_{\ell}} \}  \right)\\
	&\le \sum_{\ell=1}^{\bar{\ell}} \mathbb{P}\left( \frac{1}{\abs{S_{\ell}}}\sum_{j\in S_{\ell}}x_{j} -\hat{\mu}_{\ell}^{(12)}  > \frac{1}{8} x_* - \sqrt{\frac{2\log(2/\delta)}{n_0\phi}} \right)\\
	&\le \sum_{\ell=1}^{\bar{\ell}} \mathbb{P}\left( \frac{1}{\abs{S_{\ell}}}\sum_{j\in S_{\ell}}x_{j} -\hat{\mu}_{\ell}^{(12)}  >  \sqrt{\frac{2\log(2/\delta)}{n_0\phi}}  \right) \le \bar{\ell}\delta\ ,
\end{align*}
where we used the definition of $\Theta_{\ell}$ and $\Lambda_{\ell}$ in the third line, the assumption $h<x_{*}$ and $s<s^*$ in the fourth line and Chernoff's bound in the last line.
\end{proof}

\paragraph{Part 3: } Here, we gather all the previous lemmas to establish the three claims of the theorem.

\noindent \textbf{Conclusion for the first claim:} Lemma~\ref{lem:finpart1} leads to the first claim of Theorem~\ref{thm:up2}. 

\noindent \textbf{Conclusion for the second claim:} Let us prove that if $\epsilon^2 < d/H$ then Algorithm~\ref{algo1} with input $(\delta, \epsilon)$ outputs $\hat{k}= 1$ with probability at least $1-\delta$. By Lemma~\ref{lem:finpart1}, we have $\mathbb{P}(\hat{k}=2)\leq 0.6\delta$. Hence, it suffices to prove that $\mathbb{P}(\hat{k}=\texttt{null})\leq 0.4\delta$.
Define
\begin{equation}\label{eq:rhok}
\rho^* := \left\lfloor \log_2\left( \frac{16d}{s^* x_*^2} \right) \right\rfloor \qquad \text{ and } \qquad r^* := \left\lceil \log_2\left(\frac{1}{x_*^2}\right) \right\rceil\ .
\end{equation}
By assumption, we have $x_*^2 \le 1$, so that $r^*\geq 1$ and $\rho^*\geq 1$. Recall that $\epsilon^2 < \frac{d}{H} \le \log(2d) s^* x_*^2 \le 16\log(2d)d~2^{-\rho^*}$. If Algorithm~\ref{algo1} returns $\hat{k}=\texttt{null}$, then it implies in particular that 
Algorithm~\ref{algo3} returns $\hat{k}= \texttt{null}$ at the iteration $\rho^*$ and $r^*\leq \rho^*$. At this step, the inputs $(\delta_{r^*}, s_{r^*}, h_{r^*})$ of Algorithm~\ref{algo3} satisfy $\delta_{r^*} \le \delta/(10 r^{*3}\log(d))$, $h_{r^*} = 1/\sqrt{2^{r^*}} \le x_*$, and $s_{r^*} = 2^{r^*} d/2^{\rho^*} \le s^*$. We then deduce from Lemma~\ref{lem:finpart2} that $\hat{k}= \texttt{null}$ with probability less or equal to $\frac{3}{10 r^{*3}\log(d)}(\lceil \log_{4/3}(d/s^*) \rceil+1) \delta\leq 0.4\delta$. We conclude that $\mathbb{P}[\hat{k}=1]\geq 1-\delta$. 

\noindent \textbf{Conclusion for the third claim:}
The total number of queries made by Algorithm~\ref{algo3} with inputs $(\delta, s, h)$ is at most 
\begin{equation*}
N_{s,h} = 4096 \frac{\log_{4/3}(16d/(9s))\log(1/\delta)}{sh^2}.
\end{equation*}
Consider Algorithm~\ref{algo1} with input $(\delta, \epsilon)$. Suppose that $\epsilon^2 \ge d/H$, therefore the maximum number of iterations is less than $\rho' = \floor{\log_2(4\log(2d)d/\epsilon^2)}$. Therefore the total number of queries satisfies:
\begin{align*}
N &\le \sum_{\rho=1}^{\rho'} \sum_{r=0}^{\rho} N_{s_r, h_r}\\
&\le 4096 [\log_{4/3}(d)+2]\sum_{\rho=1}^{\rho'} \sum_{r=0}^{\rho}  \frac{\log(1/\delta_{r})}{s_rh_r^2}\\
&\le 10^4        \log(d) \log\left(\frac{10\log(d)\rho^{'3}}{\delta}\right) \sum_{\rho=1}^{\rho'} (\rho+1)2^{\rho}\\
&\le  8.10^4 \log^2(2d)\log_2\left(\frac{\log(8d\log(d))}{\epsilon^2}\right) \log\left(\frac{10\log(d)\rho'^3}{\delta}\right) \frac{d}{\epsilon^2} \\
&\le c' \log^2(2d)\log\left(\frac{\log(2d)}{\epsilon^2}\right) \log\left(\frac{2\log(2d/\epsilon^2)}{\delta}\right) \frac{d}{\epsilon^2}\ ,
\end{align*}
where $c'$ is a numerical constant.

Now suppose that $\epsilon^2 < d/H$. We have shown previously that, with probability at least $1-\delta$, the algorithm stops no later than at iteration $\rho^*$ and $r^*$ (defined in \eqref{eq:rhok}). Under such an event,  the total number of iterations satisfies 
\begin{align*}
N &\le \sum_{\rho=1}^{\rho^*-1} \sum_{r=0}^{\rho} N_{s_r, h_r} + \sum_{r=0}^{r^*}N_{s_r, h_r} \\
&\le 4096 [\log_{4/3}(d)+2] \left( \sum_{\rho=1}^{\rho^*-1} \sum_{r=0}^{\rho}  \frac{\log(1/\delta_{r})}{s_r h_r^2} +  \sum_{r=0}^{r^*} \frac{\log(1/\delta_{r})}{s_r h_r^2} \right) \\
&\le 10^4 \log(d)\log\left(\frac{10\log(d)\rho^{*3}}{\delta}\right) \sum_{\rho=1}^{\rho^*} (\rho+1) 2^{\rho}\\
&\le  10^4 \log(d)\log\left(\frac{10\log(d)\rho^{*3}}{\delta}\right)(\rho^*+1) 2^{\rho^*+1}\\
&\le 8\cdot 10^4 \log(d)\log\left( \frac{10\log(d)}{\delta}\log_2^3\left(\frac{4d}{s^*x_*^2}\right)\right) \log_2\left( \frac{8d}{s^*x^2_*} \right) \frac{d}{s^*x_*^2}\\
&\le c' \log^2(d)\log\left( H\right) \log\left( \frac{\log(H)\log(d)}{\delta} \right) H\ ,
\end{align*}
where $c'$ is a numerical constant and where we used that 
$\|x\|_2^2\leq s^* \log(2d)x_{s^*}^2$. 
\section{Full Ranking}\label{sec:full_ranking}
\subsection{Binary search algorithm}
\begin{algorithm}[!ht] 
\caption{\texttt{Binary-search}$(\delta, i,r,a,b)$  \label{algo5} }
\begin{algorithmic}
	\STATE \textbf{Input}: $\delta$, $r$, $i$ expert to be inserted, $\text{start},\text{end}$ (start/end of array $r$).
	\IF{$\text{start}=\text{end}$}
	\IF{$\texttt{compare}(\delta, 0, i, r[\text{start}]) = i$}
	\STATE \textbf{Output:} $\text{start}+1$.
	\ELSE
	\STATE \textbf{Output:} $\text{start}$.
	\ENDIF
	\ENDIF
	\IF{$\text{start}>\text{end}$}
	\STATE \textbf{Output:} $\text{start}$.
	\ENDIF
	\STATE Let $\text{mid} \gets \floor{(\text{start}+\text{end})/2}$
	\IF{$\texttt{compare}(\delta, 0, i, r[\text{mid}]) = i$}
	\STATE  \textbf{Output:} $\texttt{Binary-search}(\delta, i, r, \text{mid}+1, \text{end})$.
	\ELSIF{$\texttt{compare}(\delta, 0, i, r[\text{mid}]) = r[\text{mid}]$}
	\STATE  \textbf{Output:} $\texttt{Binary-search}(\delta, i, r, \text{start}, \text{mid}-1 )$.
	\ELSE
	\STATE \textbf{Output:} $\text{mid}$.
	\ENDIF
\end{algorithmic}
\end{algorithm}
\subsection{Proof of Theorem~\ref{thm:upn}}
For $i\in \intr{n-1}$, define:
\[
H_i := \frac{d}{\sum_{j=1}^{d}(M_{\pi(i),j} - M_{\pi(i+1),j})^2}\ .
\]
By convention, we define $H_0=0$. 
Binary insertion sort procedure makes at most $n\lceil \log_2(n)\rceil$ comparisons, hence using an union bound, we conclude that all calls to $\texttt{compare}$ algorithm output a correct result with probability at least $1-\delta$.
Moreover, inserting any expert $i$ costs at most $\lceil \log_2(n)\rceil$ calls to  $\texttt{compare}$. For any $k \in \intr{n}\setminus \{i\}$, we apply Theorem~\ref{thm:up2}. Hence, the total number of queries made by $ \texttt{Binary-search}(\delta/(n\lceil \log_2(n)\rceil), i,r,1,i-1)$ is upper bounded by 
\[
c~\log^{2}(d)\log(n)\log\left(H_i\vee H_{i-1}\right)\log[n\log(d)\log(H_i\vee H_{i-1})/\delta] (H_i\vee H_{i-1}) 
\]
with probability at least $1-\delta/n$. Here, $c$ stands for a numerical constant.   Summing for $i \in \intr{n-1}$ gives the desired result.
\section{Best expert identification}\label{sec:best_expert}

\subsection{Max search algorithm}
The max-search routine with precision $\epsilon$ is described in Algorithm~\ref{algo8} below. 
\begin{algorithm}[!ht] 
\caption{\texttt{Max-search}$(\delta, S, \epsilon)$  \label{algo8} }
\begin{algorithmic}
	\STATE \textbf{Input}: $\delta$ confidence parameter, $S$ a set of experts, $\epsilon$ a precision parameter.
	\STATE \textbf{Output}: Set $C$ of experts.
	\STATE Let $\ell = \abs{S}$ and $a_{1}, \dots, a_{\ell}$ be the elements of $S$.
	\STATE $C \gets \{a_1\}$, $m \gets a_1$.
	\FOR{$r=2,\dots, \ell$}
	\IF{$\texttt{compare}(\frac{\delta}{2|S|}, \epsilon, m, a_r) = \texttt{null}$}
	\STATE $C \gets C \cup \{a_r\}$.
	\ELSIF{$\texttt{compare}(\frac{\delta}{2|S|} , \epsilon, m, a_r) = a_r$}
	\STATE $C \gets \{a_r\}$, $m \gets a_r$.
	\ENDIF
	\ENDFOR
	\FOR {$i \in C\setminus \{m\}$ }
	\IF{$\texttt{compare}(\frac{\delta}{2|S|}, \epsilon, m, i) = m$}
	\STATE $C \gets C\cup\{i \}$
	\ENDIF
	\ENDFOR 
	\STATE \textbf{Output:} $C$.
\end{algorithmic}
\end{algorithm}

\subsection{Proof of Theorem~\ref{thm:be}}

Let $i^* \in \intr{d}$ denote the optimal expert. For $\epsilon \in (0,1)$, define $B_{\epsilon} \subset \intr{d}$ as follows:
\[
B_{\epsilon} := \{ i \in \intr{n}: \norm{M_{i^*}-M_{i}}_2^2 \le  \epsilon^2 \}.
\]
\begin{lemma}\label{lem:b0}
Consider Algorithm~\ref{algo8} with input $(\delta, S, \epsilon)$ such that $\epsilon \le 1$ and $i^* \in S$. Denote $C$ its output. With probability at least $1-\delta$, we have $i^* \in C$ and $C \subseteq B_{2\epsilon}$.
\end{lemma}
\begin{proof}
Fix $\delta \in (0,1)$ and $\epsilon \in (0,1)$. To ease notation, we  denote $R_1(a,b)$ (resp. $R_2(a,b)$), the output of $\texttt{compare}(\delta, \epsilon, a, b)$ for $a \in S$ and $b \in S$ in the first (resp. second)  loop of Algorithm~\ref{algo8}. Also, we   we write $\hat{m}$ the element with which items are compared in the second loop of Algorithm~\ref{algo8}.

Using Theorem~\ref{thm:up2}, we have that, on an event of probability $1-\delta$, all the results of $R_1(m,i)$ and $R_2(m,i)$ in Algorithm~\ref{algo8} are such that, for $s=1,2$, $R_{s}(m,i)\in \{i,\texttt{null}\}$ if $i$ is above $m$, $R_{s}(m,i)\in \{m,\texttt{null}\}$ if $m$ is above $i$, and $R_{s}(m,i)\neq \texttt{null}$ if $\|M_{m}-M_{i}\|_2\geq \epsilon$. 

Let us show that, under this event, we have $i^* \in C$ and $C \subseteq B_{2\epsilon}$. 
Indeed, if $i^*\notin C$, this implies that in the second loop we had $R_2(\hat{m},i^*)=\hat{m}$, which is not possible by definition.

Besides, we easily check that $\hat{m}$ satisfies  
$\|M_{\widehat{m}}-M_{i}\|_2\leq \epsilon$. Since $C$ only contains the elements $j$ such that we have found $R_s(\hat{m},j)\in \{j, \texttt{m}\}$ for $s=1$ or $s=2$, this implies that either $j$ is above $\hat{m}$ or that  $\|M_{\hat{m}}-M_{i}\|_2\leq \epsilon$. By triangular inequality, we have proved that   $C\subseteq  B_{2\epsilon}$.

\end{proof}

\begin{lemma}\label{lem:nbest}
Consider Algorithm~\ref{algo8} with input $(\delta, S, \epsilon)$ such that $\epsilon \le 1$ and $i^* \in S$. Denote $N_{\epsilon}$ the total number of queries made. We have 
\[
N_{\epsilon} \le c \log^2(d) \log\left(\frac{d}{\epsilon^2}\right) \log\left(\frac{|S|\log(\tfrac{d}{\epsilon^2})}{\delta}\right) \frac{d\abs{S}}{\epsilon^2},
\]
where $c$ is a numerical constant.
\end{lemma}
\begin{proof}
Observe that in Algorithm~\ref{algo8}, there are at most $2\abs{S}$ calls to the procedure $\texttt{compare}$ with precision parameter $\epsilon$. Using Theorem~\ref{thm:up2}, we get the result. 
\end{proof}

\paragraph{Conclusion:}
Fix $\delta \in (0,1)$ and denote $N$ the total number number of queries made by Algorithm~\ref{algo7}. 

Write $r=1,\ldots, \hat{r}$ for the iterations of Algorithm Algorithm~\ref{algo7} and write $S_r$ for the corresponding result of $\texttt{Max-search}$ algorithm. We write $S_0=[n]$. Applying Lemma~\ref{lem:b0}, we know that on event of probability higher than $1-\delta$, we have, 
\begin{equation}\label{eq:condition_S_r}
i\in S_r\text{ and }S_r\subseteq B_{2\cdot 4^{-r+1}}\ . 
\end{equation}
simultaneously for all $r=1,\ldots, \hat{r}$. 

We work henceforth under this event. First, this implies that $S_{\hat{r}}=\{i^*\}$. Hence, the procedure recovers the best expert.

Write $\Delta^* := \min_{i \neq i^*} \norm{M_i - M_{i^*}}_2^2$ for the minimum distance between $i^*$ and another expert. Denote $r^*=\lceil \log_4(8/\Delta^*)\rceil$. By~\eqref{eq:condition_S_r}, we have $\hat{r}\leq r^*$. By Lemma~\ref{lem:nbest}, the total number of queries made at iteration $r$ is no larger than 
\[
c \log^2(d) \log\left(\frac{d}{4^{-2r}}\right) \log\left(\frac{n 2^k\log(\tfrac{d}{4^{-2r}})}{\delta}\right) \frac{d|S_{r-1}|}{4^{-2r}}\ ,
\]
where $c$ is a numerical constant and $|S_{r-1}|\leq |B_{2\cdot4^{-r+2}}|$ for $r\geq 2$ and $|S_{0}|=n$. 

As a consequence, the 
total number $N$ of queries from iteration $2$ to $r^*$ satisfies
\begin{align*}
N&\leq c'd \log^2(d)\log\left(\frac{n\log(d)}{\delta}\right)\left[n+ \sum_{r=2}^{r^*}r \log(d4^{2r})4^{2r}|B_{2\cdot4^{-r+2}}|\right] \\
&\leq c'd \log^2(d)\log\left(\frac{n\log(d)}{\delta}\right)\left[n+ \sum_{i\neq i^*}\sum_{r=2}^{r^*}k\log(d4^{2r})4^{2r}1\{\|M_i-M_{i^*}\|_2\leq 2\cdot 4^{-r+2}\} \right]\\ 
&\leq c'' d \log^2(d)\log\left(\frac{n\log(d)}{\delta}\right)\left[n+ \sum_{i\neq i^*} \frac{1}{\|M_i-M_{i^*}\|^2_2}\log^2\left(\frac{d}{\|M_i-M_{i^*}\|_2^2\vee 1}\right) \right]\ .
\end{align*}
The result follows.

\section{Proofs of the lower bounds}\label{sec:lbproof}

\begin{proof}[Proof of Theorem~\ref{thm:lwrn}]
Fix $d\geq 1$, $\delta \in (0,1/4)$, and $\underline{\Delta}\in (\mathbb{R}^{+})^d$ such that $\|\underline{\Delta}\|^2>0$, and let $A$ be a $\delta$-accurate algorithm. Define the $2\times d$ matrix $M^*$ by $M^*_{1,j}=\underline{\Delta}_j$ for $j\in \intr{d}$ and $M^*_{2,j}= 0$. For any permutation $\pi$ of $\intr{2}$ and any permutation $\sigma$ of $\intr{d}$, we write $M^*_{\pi,\sigma}$ for the permuted matrix such that $(M^*_{\pi,\sigma})_{i,j}= (M^*_{\pi,\sigma})_{\pi^{-1}(i),\sigma^{-1}(j)}$. Obviously, we have $M^*_{\pi,\sigma}$ belongs to $\mathcal{M}_{n,d}$ and $\pi$ is the permutation that order the rows of $M^*_{\pi,\sigma}$.

For any permutations $\pi$ and $\sigma$, we write $\mathbb{P}^{(\pi,\sigma)}$ for the distribution of the data such that $X_{i,j}\sim \mathcal{N}[(M^*_{\pi,\sigma})_{i,j},1]$. We also introduce the 'null' distribution $\mathbb{P}^{0}$ such that   $X_{i,j}\sim \mathcal{N}[0,1]$. There exist only 2 permutations on $\intr{2}$, that we respectively denote $\pi_0$ (for the identity permutation) and $\pi_1$. 
Since the  strategy $A$ is $\delta$-accurate, we have, for any permutation $\sigma$ of $\intr{d}$, that  
\begin{equation}\label{eq:lower_1}
	\mathbb{P}^{(\pi_0,\sigma)} \left( \hat{\pi} \neq \pi_0 \right)\leq \delta\text{ and }  \mathbb{P}^{(\pi_1,\sigma)} \left(\hat{\pi} \neq \pi_1\right) \le \delta\enspace .
\end{equation}
Denote $N$ the total budget of the algorithm $A$. Let $x$ be the smallest integer such that 
\[
\max_{\sigma}\max_{z=0,1}\mathbb{P}^{(\pi_z,\sigma)} [N> x ]\leq \delta
\]

Next, we claim that $\mathbb{P}^{0}(N>x)\geq 1-2\delta$. Indeed, for any $\kappa>0$, consider the $2\times d$ matrix $M'$ such that $M'_{1,1}=\kappa$ and $M'_{i,j}=0$ otherwise. Write $M'_{\pi_0}$ and $M'_{\pi_1}$ for the matrix $M'$ where we have permuted the rows according to $\pi_0$ and $\pi_1$.  
Write $\mathbb{P}^{'(\pi_0)}$ and $\mathbb{P}^{'(\pi_1)}$ for the corresponding distribution of the data. Since the strategy $A$ is $\delta$-accurate, we have 
\[
\mathbb{P}^{'(\pi_0)} \left( \hat{\pi} \neq \pi_0 \right)\leq 1 -\delta \text{ and } 	\mathbb{P}^{'(\pi_1)} \left( \hat{\pi} \neq \pi_0 \right)\leq 1 -\delta
\]
As a consequence, 
\[
\mathbb{P}^{'(\pi_0)} \left( \hat{\pi} = \pi_1 \text{ and }N\leq x\right) + 		\mathbb{P}^{'(\pi_1)} \left( \hat{\pi} = \pi_0 \text{ and }N\leq x\right)\leq 2\delta 
\]
On the event where $N\leq x$, the distributions $\mathbb{P}^{'(\pi_0)}$ and $\mathbb{P}^{'(\pi_1)}$ converges in total variation distance towards $\mathbb{P}^0$ when $\kappa$ goes to zero. This implies that $\mathbb{P}^{0}[N\leq x]\leq 2\delta$. 

Consider the new algorithm $A'$ defined as follows. If the total budget of $A$ is smaller or equal to $x$, then it returns $\widehat{H}=H_1$, if the total budget is higher than, then it returns $\widehat{H}=H_0$". With a slight abuse of notation, we still write $N$ for the total budget of $A'$ and $\mathbb{P}^{(\pi,\sigma)}$ for the corresponding distributions. 
By definition of $x$, we have 
\[
\max_{\sigma}\max_{z=0,1}\mathbb{P}^{(\pi_z,\sigma)} [\widehat{H}=H_1 ]\geq 1-\delta \text{ and } \mathbb{P}^{0} [\widehat{H}=H_1 ]\geq 1-2\delta\ . 
\]
By definition of the total variance distance $TV$ between distributions (see Theorem 2.2 of \cite{tsybakov2004introduction}), we derive that 
\begin{equation*}
	\min_{z=0,1}\frac{1}{ d! } \sum_{\sigma } \textrm{TV}\left(\mathbb{P}^{(\pi_z,\sigma)}, \mathbb{P}^{0}\right) \geq 1- 3\delta
\end{equation*}
By Lemma~\ref{lem:tvkl}, we control the total variation distance in terms of Kullback-Leibler discrepancy. 
\begin{equation}\label{eq:tv0}
	\textrm{TV}\left(\mathbb{P}^{(\pi,\sigma)}, \mathbb{P}^{0}\right) \le 1 - \frac{1}{2} \exp\left\lbrace - \sum_{i=1}^{2}\sum_{j=1}^{d}  \mathbb{E}^{0}\left[ N_{i,j}\right]~ \textrm{KL}\left(\mathbb{P}^{(0)}_{i,j}, \mathbb{P}^{\pi, \sigma}_{i,j} \right)  \right\rbrace ,
\end{equation}
Under $\mathbb{P}^{\pi,\sigma}$, the distribution the $(i,j)$-th entry is $\cN(M^*_{\pi^{-1}(i),\sigma^{-1}(j)},1)$. Fixing $\pi=\pi_0$ and averaging over all permutation $\sigma$ leads to 
\[
\frac{1}{ d! } \sum_{\sigma }\exp\left[- \sum_{j=1}^d \mathbb{E}^0[N_{1,j}] \underline{\Delta}^2_{\sigma^{-1}(j)}\right] \leq 6\delta
\]
By Jensen's inequality, we deduce that 
\[
\frac{1}{ d! } \sum_{\sigma } \sum_{j=1}^d \mathbb{E}^0[N_{1,j}] \underline{\Delta}^2_{\sigma^{-1}(j)} \geq  \log(1/(6\delta))\ .
\]
Arguing similarly for the permutation $\pi_1$, we arrive at 
\[
\frac{1}{ d! } \sum_{\sigma } \sum_{j=1}^d \mathbb{E}^0[N_{1,j}+N_{2,j}] \underline{\Delta}^2_{\sigma^{-1}(j)} \geq  2\log(1/(6\delta))\ .
\]
Then, by Lemma~\ref{lem:opt}, this left-hand-side term of this equation is larger or equal to $\mathbb{E}^0[N]\|\underline{\Delta}\|_2^2/d$. Since $\mathbb{E}^0[N]\geq (1-2\delta)x$, we conclude that 
\begin{equation*}
	x \geq \frac{d\log(1/6\delta)}{2\|\underline{\Delta}\|_2^2}\ ,
\end{equation*}
which conclude the proof. 

\end{proof}

\begin{proof}[Proof of Theorem~\ref{thm:lwrn_minimax}] 
We introduce the $n\times d$ matrix $M_0$ by $(M_0)_{1,j}= 1$ and $(M_0)_{i,j}= 1-\sum_{l=1}^{i-1}\Delta_l/\sqrt{d}$ for $i=2,\ldots, n-1$. All the rows of $M_0$ are constant. Obviously, the true ranking $\pi$ is the identity, while $\|M_{i}-M_{i+1}\|_2^2= \Delta_i^2$. Given $i\in \intr{n_0-1}$, let $\pi_{i,i+1}$ denote the transposition that exchanges $i$ and $i+1$. Any $\delta$-accurate algorithm $A$ is able to decipher with probability higher than $1-\delta$ between the matrix $M_0$ and the permuted matrix $(M_0)_{\pi_{i,i+1}}$, which, since the rows of $M_0$ are constant, is equivalent to best arm identification in a two-arm problem with gap $\Delta_i/\sqrt{d}$. Assume that we observe the matrix $M_0$ with a standard Gaussian noise and denote $\mathcal{B}_0$ for the corresponding distribution. By~\cite{kaufmann2016complexity}, we have $\mathbb{E}_{\mathcal{B}_0,A}[N_{i+1}]\geq d\log[1/4\delta)]/\Delta_i^2$. By linearity, we deduce that 
\[
\mathbb E_{\mathcal{B}}[N]\geq \sum_{i=1}^{n-1}\frac{d}{\Delta_i^2}\log[1/(4\delta)]\ , 
\]
which concludes the proof.
\end{proof}

\begin{proof}[Proof of Theorem~\ref{thm:lwrn_minimax_best_arm}]
This theorem is a straightforward consequence of existing lower bounds in multi-armed bandits for best arm identification. Indeed, the set $\overline{\mathbb{D}}^{n,d}_{\bm{\Delta}}$ contains in particular problem instances that are constant over questions. For these instances, our $d$-dimensional problem is akin to a one dimensional problem, i.e.~to a standard multi-armed bandit problem. Existing lower bounds in this case imply the bounds, see e.g.~\cite{kaufmann2016complexity}.

\end{proof}

\section{Technical results}
The first lemma is a slight generalization of Lemma 3.1 in \citealp{castro2014adaptive}
\begin{lemma}\label{lem:opt} 
Let $\Upsilon_d$ denote the set of permutations on $\intr{d}$. Consider any vectors $\bm{x} = (x_1, \dots, x_d) \in \mathbb{R}_+^d$ and any $\bm{b} = (b_1, \dots, b_d) \in \mathbb{R}_+^d$.
\begin{equation*}
	\sum_{\pi \in \Upsilon_d}\sum_{i=1}^d b_i ~x_{\pi(i)} = \frac{\abs{\Upsilon_d}}{d} \sum_{i=1}^{d}x_i\sum_{i=1}^{d}b_i\ . 
\end{equation*}
\end{lemma}

\begin{proof}
Fix any such $\bm{b}$ and $\bm{x}$. We simply exchange the summation order. 
\begin{align*}
	\sum_{\pi \in \Upsilon_d}\sum_{i =1}^d b_i~x_{\pi(i)} &=   \sum_{i =1}^d b_i~\sum_{\pi \in \Upsilon_d}x_{\pi(i)} 
	= \sum_{i =1}^d b_i~\frac{\abs{\Upsilon_d}}{d}\sum_{i=1}^d x_{i}\\
	&= \frac{\abs{\Upsilon_d}}{d} \sum_{i=1}^{d}x_i \sum_{i=1}^{d}b_i \  .
\end{align*}
\end{proof}

\begin{lemma}[\citealp{kaufmann2016complexity}, with slight modification]\label{lem:tvkl}
Let $\nu$ and $\nu'$ be two collections of $d$ probability distributions on $\mathbb{R}$, such that for all $a\in \intr{d}$, the distributions $\nu_a$ and $\nu_{a'}$ are mutually absolutely continuous. For any almost-surely finite stopping time $\tau$ with respect to the data collected before $\tau$,
\[	
\sup_{\cE \in \cF_{\tau}} \abs{\mathbb{P}_{\nu}\left(\cE\right) - \mathbb{P}_{\nu'}\left(\cE\right)} \le 1-\frac{1}{2} \exp\left\lbrace-\sum_{a=1}^{d} \mathbb{E}_{\nu}\left[N_{a}(\tau)\right] \text{KL}(\nu_a, \nu'_a)\right\rbrace.
\] 
\end{lemma}

\begin{lemma}\label{lem:sparse}
Let $(x_i)_{i \in \intr{d}}$ denote a sequence of non-increasing numbers in $[0,1]$. We have
\begin{equation}
	\sum_{i=1}^{d} x_i \le \log(2d)~\max_{1\leq s \le d}\{ s~x_s \}.
\end{equation} 
\end{lemma}
\begin{proof}
Suppose for the sake of contradiction that, for any $s\in \intr{d}$, we have 
\begin{equation*}
	s\log(2d) ~x_s < \sum_{i=1}^{d} x_i.
\end{equation*}
Dividing the above inequality by $s$ and summing it over $s \in \intr{d}$, we deduce that 
\begin{equation*}
	\log(2d)\sum_{s=1}^{d}x_s < \sum_{s=1}^{d}\frac{1}{s}~\sum_{i=1}^{d}x_i, 
\end{equation*}
which contradict the upper bound on the partial sum of harmonic series $H_d = \sum_{s=1}^{d} \frac{1}{s}$:
\[
H_d \le \log(d)+\gamma + \frac{1}{2d-1}\ .
\] 
\end{proof}

The following lemma is a direct consequence of Chernoff's inequality applied to a binomial random variable.
\begin{lemma}\label{lem:cher}
Let $X = \sum_{i=1}^{n}X_i$, where $X_i$ are independent and follow the Bernoulli distribution with parameter $p$, and let $\mu = \mathbb{E}[X]$, we have for any $\kappa \ge 0$
\[
\mathbb{P}\left( X \ge (1+\kappa)\mu \right) \le \frac{e^{\kappa \mu}}{(1+\kappa)^{(1+\kappa)\mu}}.
\]
Besides, for any $\kappa\in (0,1)$, we have
\begin{align*}
	\mathbb{P}\left( X \le (1-\kappa) \mu\right) &\le \exp\left(-\frac{\kappa^2 \mu}{2}\right)\\
	\mathbb{P}\left( X \ge (1+\kappa) \mu\right) &\le \exp\left(-\frac{\kappa^2 \mu}{3}\right)
\end{align*}

\end{lemma}
\begin{proof}
We only show the first inequality, the other ones being classical. All of them are consequences of Chernoff inequality.
For a variable $Y$, we denote $M_Y(t)$ its moment generating function.
Recall that, for any $X_i \sim \text{Ber}(p)$, we have $M_{X_i}(t) \le e^{p(e^t-1)}$. Moreover for $X = \sum_{i=1}^{n}X_i$, we have
\begin{align*}
	\mathbb{P}\left(X \ge k \right) &\le \min_{t>0} \frac{M_X(t)}{e^{(k)}} = \min_{t>0} \frac{\prod_{i=1}^{n}M_{X_i}(t)}{e^{tk}}\\
	&\le \min_{t>0} \frac{\left(e^{p(e^t-1)}\right)^n}{e^{tk}} =  \min_{t>0} \frac{e^{\mu(e^{t}-1)}}{e^{tk}}\ .
\end{align*}
For $k = (1+\kappa) \mu$, we take $t = \log(1+\kappa)$ and we obtain
\begin{align*}
	\mathbb{P}\left( X \ge (1+\kappa) \mu\right) &\le \frac{e^{\mu\left(e^{\log(1+\kappa)} - 1\right) }}{e^{(1+\kappa)\mu \log(1+\kappa)}}
	= \frac{e^{\kappa \mu}}{(1+\kappa)^{(1+\kappa)\mu}}\enspace .
\end{align*}
\end{proof}

\end{document}